\documentclass[letterpaper]{article} % DO NOT CHANGE THIS
\usepackage{aaai24}  % DO NOT CHANGE THIS
\usepackage{times}  % DO NOT CHANGE THIS
\usepackage{helvet}  % DO NOT CHANGE THIS
\usepackage{courier}  % DO NOT CHANGE THIS
\usepackage[hyphens]{url}  % DO NOT CHANGE THIS
\usepackage{graphicx} % DO NOT CHANGE THIS
\usepackage{natbib}  % DO NOT CHANGE THIS AND DO NOT ADD ANY OPTIONS TO IT
\usepackage{caption} % DO NOT CHANGE THIS AND DO NOT ADD ANY OPTIONS TO IT
\usepackage{algorithm}
\usepackage{algorithmic}
\usepackage{amsmath}
\usepackage{amsthm}
\usepackage{booktabs}
\usepackage{subfigure}
\usepackage{appendix}
\usepackage{multicol}
\usepackage{color}
\usepackage{amssymb}
\newcommand{\cor}[1]{\textcolor{blue}{#1}}

\usepackage{newfloat}
\usepackage{listings}
\DeclareCaptionStyle{ruled}{labelfont=normalfont,labelsep=colon,strut=off} % DO NOT CHANGE THIS
\floatstyle{ruled}
\newfloat{listing}{tb}{lst}{}
\floatname{listing}{Listing}
\title{Refining Latent Homophilic Structures over Heterophilic Graphs for \\Robust Graph Convolution Networks}
\author{
    Chenyang Qiu\textsuperscript{\rm 1},
    Guoshun Nan\textsuperscript{\rm 1}\thanks{Guoshun Nan is the corresponding author.},
    Tianyu Xiong\textsuperscript{\rm 1},
    Wendi Deng\textsuperscript{\rm 1},
    Di Wang\textsuperscript{\rm 1},
    Zhiyang Teng\textsuperscript{\rm 2},
    Lijuan Sun\textsuperscript{\rm 1},
    Qimei Cui\textsuperscript{\rm 1},
    Xiaofeng Tao\textsuperscript{\rm 1}}
\affiliations{
 \textsuperscript{\rm 1}Beijing University of Posts and Telecommunications, China\\
 \textsuperscript{\rm 2}Nanyang Technological University, Singapore\\
 \{cyqiu, nanguo2021, tyxiong, dengwendi, wdwdwd, sunlijuan, cuiqimei, taoxf\}@bupt.edu.cn, chihyangteng@gmail.com
}

\usepackage{bibentry}
\begin{document}

\maketitle

\begin{abstract}
Graph convolution networks (GCNs) are extensively utilized in various graph tasks to mine knowledge from spatial data. Our study marks the pioneering attempt to quantitatively investigate the GCN robustness over omnipresent heterophilic graphs for node classification. We uncover that the predominant vulnerability is caused by the structural out-of-distribution (OOD) issue. This finding motivates us to present %LHS, \cor{\textbf{R}obust \textbf{S}tructure learning of \textbf{G}raph \textbf{C}onvolution \textbf{N}etworks,}
a novel method that aims to harden GCNs by automatically learning \textbf{L}atent \textbf{H}omophilic \textbf{S}tructures over heterophilic graphs. We term such a methodology as \textbf{LHS}. To elaborate, our initial step involves learning a latent structure by employing a novel self-expressive technique based on multi-node interactions. Subsequently, the structure is refined using a pairwisely constrained dual-view contrastive learning approach. We iteratively perform the above procedure, enabling a
GCN model to aggregate information in a homophilic way on
heterophilic graphs. Armed with such an adaptable structure, we can properly mitigate the structural OOD threats over heterophilic graphs. Experiments on various benchmarks show the effectiveness of the proposed LHS approach for robust GCNs.
\end{abstract}
%%Such an crucial observation is derived from our quantitative analysis of structural OOD with the proposed peripheral heterophily.

%We further provide the distribution `left shift' evidence on heterophilic structure, that is, the distribution concentrates more on low local homophily ratio values.

%The peripheral nodes connect to other class node, as shown in the pink circle. Peripheral heterophily quantitatively depicts the peripheral nodes (edges) proportion in the peripheral structure (pink circle). It is only need to change the peripheral structure permutation on (a) (highly homophilic graph), while without changing other graph settings, and (a) can gradually evolve into (c) (highly heterophilic graph). It reveals that the peripheral heterophily changing is the key factor causing distribution shift and increase structural OOD risks.

\section{Introduction}
\label{sec:intro}
%\vspace{-5pt}
\begin{figure}[t]
  \centering
  \includegraphics[scale=0.35]{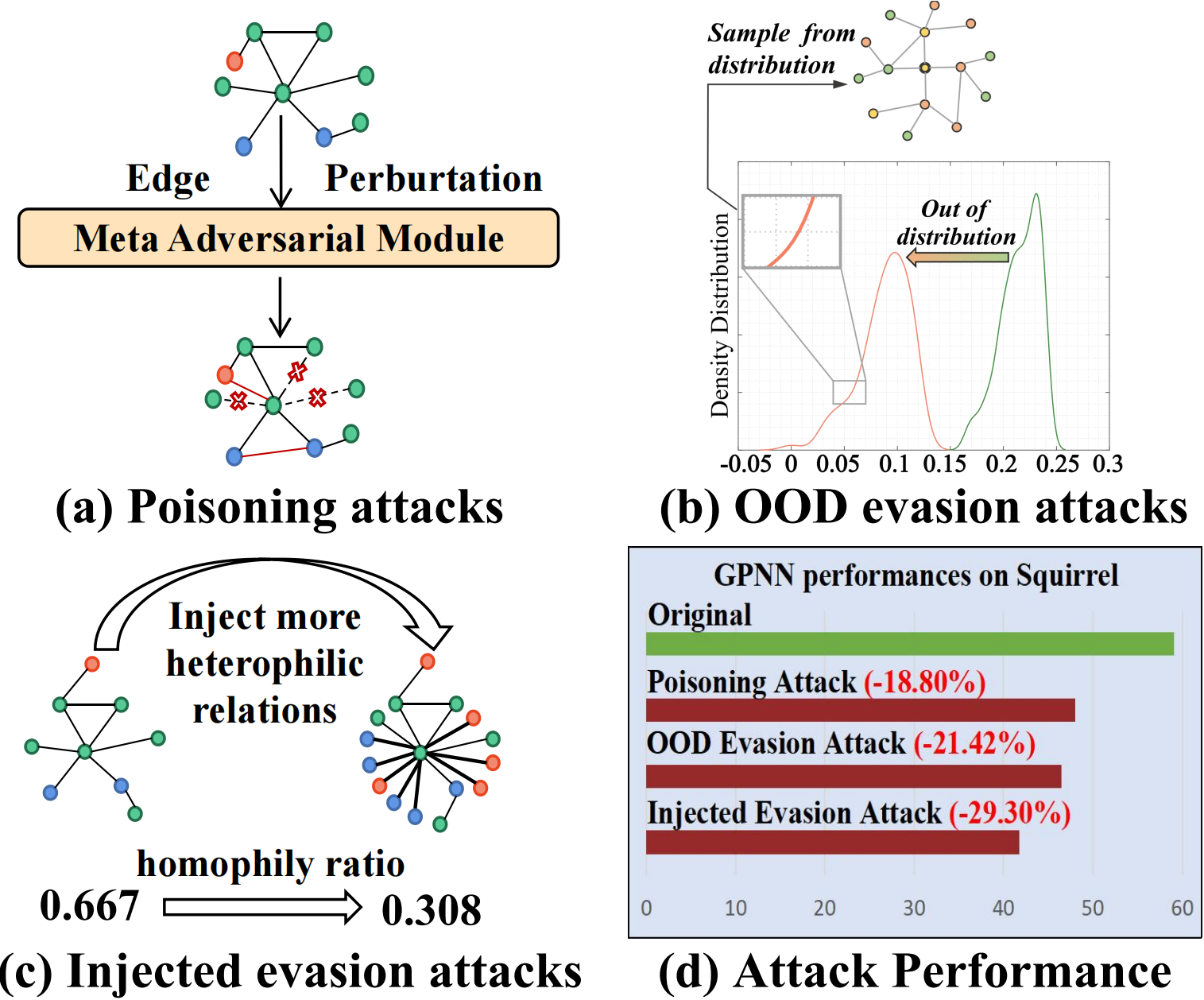}\\
  \vspace{-2mm}
  \caption{Illustration for the vulnerability of existing GPNN under various threats on Squirrel, including a poisoning attack, and another two adapted from evasion attacks. The subfigures (a), (b), and (c) depict how we generate the above three attacks, and (d) reports the significant performance degradation under each attack.}\label{fig:intro-attack}
  \vspace{-1em}
\end{figure}
Graph-structured spatial data, such as social networks~\cite{qiu2022vgaer} and molecular graphs, is ubiquitous in numerous real-world applications~\cite{li2022fast}. Graph convolution networks (GCNs)~\cite{kipf2017semisupervised}, following a neighborhood aggregation scheme, are well-suited to handle these relational and non-Euclidean graph structures, and have been widely applied in various graph tasks, including node classification and recommender systems.
Recently, there has been a surge in GCN approaches for challenging heterophilic graphs~\cite{zhu2020beyond}, where most neighboring nodes have different labels or features. These methods can be divided into two categories: 1) Multi-hop-based approaches~\cite{abu2019mixhop,jin2021universal,TDGNN}; 2) Ranking-based approaches~\cite{liu2021non,yuan2021node2seq,yang2022graph}. The former group learns node representations based on multi-hop aggregations, while the latter performs selective node aggregations by a sorting mechanism. These GCN methods continue to advance the state-of-the-art performance for node classification and have enabled various downstream applications~\cite{pmlr-v139-lin21d,qiu20233d}.

\begin{figure*}[!t]
  \centering

\includegraphics[scale=0.44]{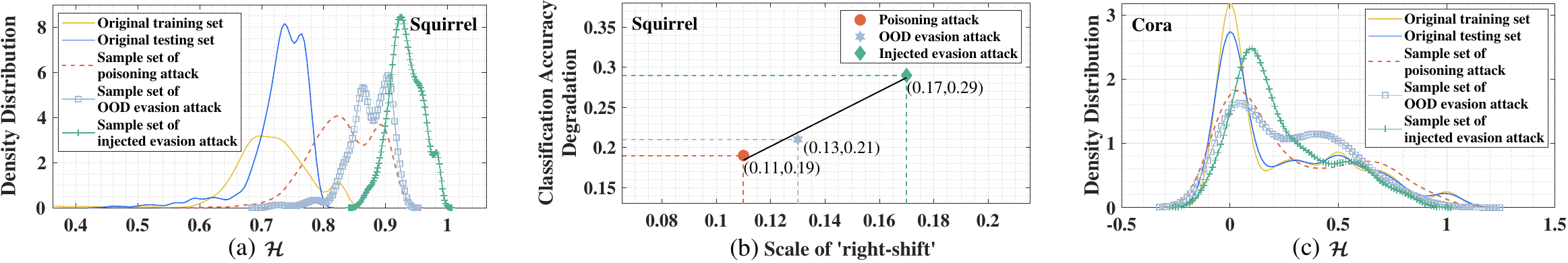}\\
  \vspace{-12pt}
  \caption{Illustration of ${\mathcal {H}}$ distributions and the ``right-shift''. (a) ${\mathcal {H}}$ distributions of various data over Squirrel, including crafted sample sets of three attacks.
  (b) correlation between the ``right-shift'' of ${\mathcal {H}}$ distributions and node classification degradation.
  (c) ${\mathcal {H}}$ distributions of various data, including the original train and test set of a homophilic dataset Cora, and crafted sample sets of three attacks.}\label{fig:intro-analyze}
  \vspace{-1em}
\end{figure*}

%%  The 'left-shift' on cross-domain (a) and inner-domain distributions of heterophilic graph Squirrel (b) are obvious (especially for high-layer situation), that is, structural distribution concentrates more on low values. Abundant evidence reveals that this distribution shift is very common on heterophilic structure. However, the inner-domain of  graph Cora has less structural OOD risks (c).

Despite the significant success of the current GCN methods on heterophilic graphs, these approaches are \textbf{extremely vulnerable to malicious threats} that aim to distort the structure of the target graph during testing. We conduct experiments to attack the state-of-the-art GPNN~\cite{yang2022graph} method, which was trained on the popular Squirrel~\cite{Pei2020Geom-GCN:} benchmark for heterophilic graphs, using samples created by various attacks. Fig.\ref{fig:intro-attack} demonstrates that the accuracy of node classification can be greatly reduced under three different types of destructive attacks, including a well-known poisoning attack~\cite{jin2020graph}, and two attacks adapted from evasion attacks ~\cite{biggio2013evasion,Adversarial_Feature_Selection_against_Evasion_Attacks}.  Specifically, as shown in Fig.\ref{fig:intro-attack} (a), the poisoning attack produces adversarial structural perturbations to the edges of the graph, fooling GPNN to make incorrect predictions. The proposed two evasion-based attacks are referred to as ``OOD evasion attacks'' and ``injected evasion attacks'', respectively. Fig.\ref{fig:intro-attack} (b) and Fig.\ref{fig:intro-attack} (c) demonstrate how the sample sets for these two attacks are created. The first generates a graph with a node distribution that is vastly different from that of the target testing set, while the second manipulates the target graph by injecting more heterophilic edges. Under these three attacks, Fig.\ref{fig:intro-attack} (d) shows that  classification accuracy of GPNN is sharply decreased by $18.90\%$, $21.42\%$, and $29.30\%$, respectively.

% indicating that the graph structure is nontrivial for GCN robustness.  %These results show the demand of developing a robust GCN method over heterophilic graphs.

%We take the first step towards quantitatively analyzing the robustness of GCN approaches over omnipresent heterophilic graphs for node classification, and reveal that the vulnerability under various attacks is mainly attributed to structural out-of-distribution (OOD).

To analyze the reasons why GCN methods are fragile on heterophilic graphs, we further depict the ${\mathcal {H}}$ distributions~\cite{zheng2022graph} of the crafted data from the aforementioned three attacks, as well as the distributions of the original train and test sets of Squirrel in Fig.\ref{fig:intro-analyze} (a). Here $ {\mathcal {H}}$ represents the node-level heterophily, which is the proportion of a node's neighbors that have a different class\footnote{A more formal definition is given in Preliminaries Section, and higher $ {\mathcal {H}}$ values indicate a node with strong heterophily.}.
Fig.\ref{fig:intro-analyze} (a) demonstrates that the distributions of three attack samples are all located to the right of the training set, with the most destructive sample for the GPNN method being the furthest to the right. This observation led us to investigate the correlation between the ``right-shift'' of the ${\mathcal {H}}$ distribution relative to the train set and the vulnerability of GCN approaches. This correlation is visualized in Fig.\ref{fig:intro-analyze} (b) and it is shown that the scale of ``right-shift'' is strongly proportional to the degradation of node classification performance. We refer to this phenomenon as ``\textbf{structural out-of-distribution (OOD)}'' in GCN methods for graphs of spatial data.

To investigate the underlying cause of the aforementioned structural OOD, we attacked another GPNN model that was trained on the homophilic graph Cora~\cite{yang2016revisiting} and depicted the resulting ${\mathcal {H}}$ distributions in Fig.\ref{fig:intro-analyze} (c). Interestingly, the shifts of the three attacks relative to the training set of Cora are very small. This minor ``right-shift'' enables the GPNN model trained on Cora to be more robust\iffalse(see further details in \cor{Appendix C})\fi. We attribute this to the strong homophily present in the Cora dataset and believe that more homophily will result in less ``right-shift'' under attacks, even for heterophilic graphs, and hence alleviate the structural OOD.

In light of the above discussion, a critical question arises: ``\textit{How can a GCN model automatically learn an appropriate \textbf{homophilic structure} over \textbf{heterophilic graphs} to reduce
the scale of ``right-shift'' in ${\mathcal {H}}$ distributions?}'' This could help to make the model more resistant to malicious attacks on heterophilic graphs.
Achieving this goal is challenging. Despite the success of many structure learning-related methods~\cite{jin2021node,jin2021universal,he2022block}, they also tend to strengthen the heterophily or only focus on the \emph{local} relations between two nodes rather than considering the \emph{global} connections. These methods still suffer from vulnerability issues under attacks (as seen in Figure \ref{fig:meta} and Table \ref{tab:exp-ood-injected}), and they are hardly able to address the challenge.

% where the heterophilic graph may consist of thousands of nodes, along with hundreds of thousands of edges?}'' e.g., total $2,089$ nodes and $217,073$ edges are involved in Squirrel.

%Specifically, SimP-GNN introduced a feature similarity-preserving mechanism that can adaptively integrate graph structure and node features. UGCN relied on $k$-hop neighbors to learn graph structure and BM-GCN incorporates block modeling into the aggregation process for better graph aggregation. All the above structure learning GCN methods don't explicitly consider reducing the scale of ``right-shift'' of ${\mathcal {H}}$ distributions, and are vulnerable to three attacks discussed in Figure \ref{fig:intro-attack} (see details in Section \ref{sec:robust-exp}).

We address the above challenging question with a novel method called LHS. The key components of the proposed LHS are: 1) a self-expressive generator that automatically induces a latent homophilic structure over heterophilic graphs via multi-node interactions, and 2) a dual-view contrastive learner that refines the latent structure in a self-supervised manner. LHS iteratively refines this latent structure during the learning process, enabling the model to aggregate information in a homophilic way on heterophilic graphs, thereby reducing the ``right-shift'' and increasing robustness\iffalse \footnote{We will release our code to the research community}\fi. It should be noted that the original graph
Experiments on five benchmarks of heterophilic graphs show the superiority of our method. We also verify the effectiveness of our LHS on three public homophilic graphs. Additionally, the induced structure can also be applied to other graph tasks such as clustering. Our contributions are as follows:
\vspace{-4pt}
\begin{itemize}
\item We quantitatively analyze the
robustness of GCN methods over omnipresent
heterophilic graphs for node classification, and reveal that the ``right-shift'' of ${\mathcal {H}}_{node}$ distributions is highly proportional to the model's vulnerability, i.e., the structural OOD. To the best of our knowledge, this is the first study in this field.
%the model that  have revealed the inherent structural instabilities of heterophilic graphs via structural OOD formulations, which leads to the GCN robustness problems.
\item We present LHS, a novel method that strengthens GCN against various attacks by learning latent homophilic structures on heterophilic graphs.
%reduce the ``right-shift''
%heterophilic structural learning and devise LHS, which enables to learn a robust structure and reduce structural OOD risks over heterophilic graphs.
\item We conduct extensive experiments on various spatial datasets to show the effectiveness of the proposed LHS in mitigating the structural OOD issue.
%Both the quantitative and qualitative results show that our method not only benefits the robustness of GCN methods on heterophilic graphs but also improves the stability of the ones on homophilic graphs.
%design structural OOD attacks to simulate the potential risks from heterophilic environment, experiments against multiple structural attacks show the robustness of our proposed method.
\end{itemize}
\vspace{-10pt}

\section{Related Work}
\subsection{Graph Convolution Networks}
There is a line of early studies in graph convolution networks (GCNs) \cite{kipf2016variational,kipf2017semisupervised,GraphSAGE,veličković2018graph}. Recent GCN approaches over heterophilic graphs can be grouped into multi-hop-based ones \cite{abu2019mixhop,zhu2020beyond,jin2021node,TDGNN,wang2022acmp}, ranking-based ones~\cite{liu2021non, wang2022powerful,yang2022graph}, and the ones using GCN architecture refinement~\cite{FAGCN,DMP,WRGNN,yan2021two,ACM, xu2023node,li2023restructuring,zheng2023finding}. These methods have achieved remarkable success in graph node classification. However, robustness is yet to be explicitly considered on challenging heterophilic graphs.

\subsection{Robust Graph Convolution Networks}
Recently we have witnessed a surge in the robustness of GCN on heterophilic graphs. These methods can be categorized into the structure learning-based ones ~\cite{jin2020graph,jin2021universal,he2022block,zhu2022does,liu2023beyond}, and the ones based on adversarial training~\cite{dai2018adversarial,zhu2019robust,zhang2020self,GNN-Guard,AD-GCL}. The most related to our work is ProGNN~\cite{jin2020graph} which explores the low-rank and sparsity of the graph structure, and SimP-GCN~\cite{jin2021node} which relies on a similarity preservation scheme for structure learning. Our work differs from the above methods in two aspects: 1) We focus on the structural OOD issue of GCN approaches over heterophilic graphs. To the best of our knowledge, this problem is largely ignored in previous works. 2) We iteratively refine the latent structure of heterophilic graphs by a novel self-expressive method and a dual-view contrastive learning scheme, enabling a GCN model
to effectively aggregate information in a homophilic way on heterophilic graphs.

%IDGL,jin2021node
\begin{figure*}[!t]
  \centering
  %17,7
  \includegraphics[width=17cm,height=5.5cm]{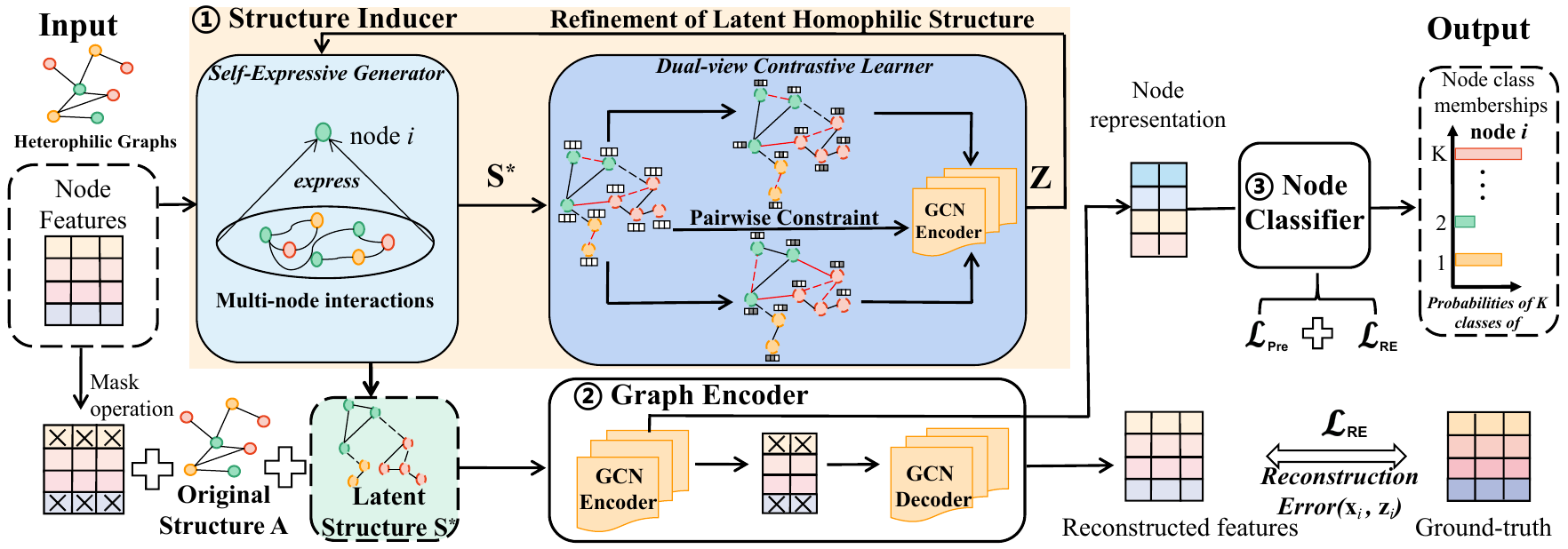}\\
  \vspace{-10pt}
  \caption{The architecture of the proposed LHS, which consists of three modules: structure inducer, graph encoder, and node classifier. The key ingredient structure inducer involves two components, i.e., the self-expressive generator and dual-view contrastive. The former learns a latent homophilic structure by multi-node interactions and then the latter refines the structure. We iteratively perform such a procedure to learn a better homophilic structure on heterophilic graphs. The refined structure will be fed to the graph encoder for representation aggregation, and finally, the classifier for node classification.}\label{fig:model}
  \vspace{-15pt}
\end{figure*}
% relies on the updated representations

\section{Preliminaries}
We denote the graph as $G=(V,E,\mathbf {X})$, where $V \in \mathbb{R}^{N \times N}$ is the set of nodes, $E$ is the set of edges between nodes, $\mathbf{X}$ is the node features and $(V,E)$ can form the original network structure $\mathbf A$. We aim to generate a latent homophilic structure for robust GCN in node classification. For convenience, we give the following edge definition:

\newtheorem{definition}{Definition}
\begin{definition}
(Positive/Negative Edge)
A \textbf{positive} edge indicates that the two nodes in the link have the same type, while \textbf{negative} one refers to a link that connects two nodes with different types.
%The positive edge refers to an edge whose two nodes belong to the same class, while the negative one indicates an edge whose two nodes belong to the different classes.
\end{definition}

%\vspace{-5pt}
%to deal with the error I delete the '\nonident' before the '\textbf'
\noindent
\textbf{Node-level Heterophily:}
\label{sec:peri-heter}
We use $ {\mathcal {H}}$ to represent the node-level heterophily, which is the proportion of a node's neighbors that have a different class. We refer to ~\cite{zheng2022graph} to give a formal definition, and it is a fine-grained metric to measure the edge heterophily in a graph.

\newtheorem{remark2}{Remark}

\begin{definition}
	\label{Def2}
Node-level Heterophily ($\mathcal {H}$):
\begin{equation}
\mathcal {H}(v_i)=\frac{|(v_i,v_j)\mid y(v_i)\neq y(v_j)|} {|\mathcal {E}(v_i)|}, \forall{v_i, v_j} \in V,
\end{equation}
\end{definition}
where $(v_i,v_j) \in \mathcal {E}(v_i)$, $\mathcal E (v_i)$ is the edge set of $v_i$, $y(v_i)$ is the node class of $v_i$, and $|\cdot|$ represents the number of edges.
The nodes with strong heterophily have higher $\mathcal {H}$ (closer to 1), whereas nodes with strong homophily have smaller $\mathcal {H}$ (closer to 0). It also provides an edge distribution sampling set for quantitative analysis over heterophilic graphs.

\noindent\textbf{OOD Formulation:} We rigorously formulate the ego-graph edge distribution by utilizing the proposed node-level heterophily, and the formulation further enables the multi-layer edge distribution analyses. The `right-shift' phenomenon found in the heterophilic graphs also motivates us to propose latent homophilic structure refinement. Theoretical analysis from a spectral-domain view is given in Appendix 1~\footnote{Appendices are available in the preprint version.} to further elaborate the rationale of the proposed LHS.
%\cor{find} structural OOD empirical evidence.

%which intuitively depicts the existing GCN models faced structural robustness problems.

%us raise concerns about exacerbating structural robustness risk.

%exacerbates the structural robustness risk.

%makes us raise robustness concerns about cross and inner-domain OOD problem in heterophilic graphs.
%This section further provides support for our concerns about the model robustness and generalization.
%, where $\bar {\mathcal {H}}_{EP}(A_{v_i})= 1-\mathcal {H}_{EP}(A_{v_i})$ depicts the useful structure information for node classification.
\iffalse
\begin{figure*}[!t]
\vspace{-5pt}
  \centering
  \includegraphics[width=18cm,height=7cm]{model4.png}\\
  \caption{1}\label{1}
\end{figure*}

\begin{figure*}[!t]
  \centering
  \includegraphics[width=17cm,height=7cm]{figures/model13.pdf}\\
  \vspace{-10pt}
  \caption{The architecture of RSGCN, which consists of three modules: robust structure learner, remasked feature augmenter and node classifier.}\label{fig:model}
  \vspace{-15pt}
\end{figure*}
\fi

\noindent\textbf{Edge Distribution Formulation:} Given a random node $v_i \in V$, we define $v_i$'s $k$-hop neighbors as $N_{v_i}(k)$ (where $k$ is an arbitrary positive integer) and the nodes in $N_{v_i}(k)$ form an ego-graph substructure called $A_{v_i}(k)$, which consists of a local adjacency matrix represented as $A_{v_i}(k)=\left\{a_{v_i u} \mid u \in N_{v_i}(k)\right\}$. In this way, we can study the distribution of the $k$-hop substructure $A_{v}$ via $p( {\mathcal {H}} \mid A_{v_i}(k))=p( {\mathcal {H}} \mid A_{v_i}(1)A_{v_i}(2)... A_{v_i}(k))$. It's worth noting that the ego-graph can be seen as a Markov blanket for the centered node $v_i$, meaning that the conditional distribution $p( {\mathcal { H}} \mid A_{v_i}(k))$ can be decomposed as a product of independent and identical marginal distributions $p( {\mathcal {H}} \mid A_{v_i}(i), i \in k)$ for each of the ${A_{v_i}(j), j \leq k}$. We also provide more empirical observations about the ``right-shift'' phenomenons on heterophilic graphs, which are available in Appendix 3.

\vspace{-5pt}
\section{Methodology}
\subsection{Overview}
In this section, we present the proposed LHS. Our goal is to learn an appropriate latent homophilic structure from heterophilic graphs, so as to reduce the scale of ``right-shift'' in $\mathcal{H}$ distributions. Inspired by the analysis in the Introduction Section that more homophily of a graph can reduce the ``right-shift'', our latent structure tends to encourage positive edge connections by increasing the edge weights for pairs of nodes with the same type, and suppresses negative edge connections by reducing the edge weight for nodes with different types. Fig.~\ref{fig:model} shows the architecture of LHS.

\subsection{Structure Inducer}
The proposed structure inducer involves a self-expressive generator and a dual-view contrastive learner.
%The former aims to generate a latent homophilic structure and the latter aims to further refine this homophilic structure.
% without using  structure $A$
%\
\iffalse
\cor{The self-expressive technology has been successfully applied in computer vision for object detection and segregation~\cite{zhang2019self}, while has not been applied in graph structure learning. }
\fi

\subsubsection{Self-Expressive Generator.}

Our proposed self-expressive generator produces a latent homophilic structure over heterophilic graphs in three steps:

\noindent
\textbf{Step 1: Capturing multi-node information.} Given the node features $\mathbf X$, we aim to capture the multi-node feature information by expressing one node feature via a linear or affine combination of other node features. Differently from the existing structure learning method with pair-wise similarity matrix~\cite{jin2021universal}, our inducer can generate fine-grained latent structure $S^* \in \mathbb{R}^{N \times N}$ by discovering the global information in low-dimension subspace. Specifically, for $\forall v_i \in V$, we express it by a linear sum of multi-node features $\mathbf x_{j}, v_j \neq v_i$, which can be expressed as $\mathbf x_{i}=\sum_{v_j \in V} q_{i j} \mathbf x_{j}$, where $q_{i j}$ is the $(i, j)$ th element of a coefficient matrix $Q$.
%, which has been successfully applied in computer vision for object detection and segregation, while has not been applied in structure learning
%Compared with the popular pairwise metrics such as cosine or kNN network~\cite{jin2021node}, multi-node information capturing provides a broader view for structure learning, because the former directly generate structure by only conducting a naive feature-pair inner product.

%This multi-node information capturing provides a broader view for latent homophihlic structure learning, compared with the popular pairwise metrics such as cosine or kNN network, because the pairwise methods directly generate structure by calculating a naive feature-pair inner product.

\noindent
\textbf{Step 2: Optimizing the generator loss.} We use the coefficient matrix $Q$ to generate latent structure. The optimization problem to solve $Q$ can be formulated as follows:
\begin{equation}\label{self}
\min _{Q}\|Q\|_{F} \quad s.t. \quad \mathbf X=Q \mathbf X ; \ \rm {diag}(Q)=0
\vspace{-5pt}
\end{equation}
where $\|Q\|_{F}$ is the Frobenius matrix norm~\cite{bottcher2008frobenius} of $Q$ and $\rm {diag}(Q)$ denotes the diagonal entries of $Q$. Eq.~\ref{self} optimizes a block-diagonal matrix $Q$ to generate the latent structure $S^*$. Each block of $Q$ contains the nodes which belong to the same class, thus mitigating the ``right-shift" phenomenon. We relax the hard constraint $\mathbf X=Q \mathbf X$ with a soft constraint $(\mathbf X-Q \mathbf X)$, as the exact reconstruction of $\mathbf X$ may be impractical. The relaxation formulation is:
\begin{equation}
\min _{Q} \mathcal{L}_{S E}=\|\mathbf X-Q \mathbf X\|_{F}^{2}+\lambda_{1}\|Q\|_{F}^{2},  s.t. \rm{diag}(Q)=0,
\vspace{-5pt}
\end{equation}
where $\lambda_{1}$ is a weight hyperparameter of optimization.

\noindent
\textbf{Step 3: Generating latent homophilic structure.}
We construct the latent homophilic structure $S^*$ by $Q+Q^{T}$, while this structure still has noise and outliers. Therefore, we rely on Algorithm~\ref{algo} to generate $S^*$. Specifically, the SVD decomposition in Algorithm~\ref{algo} aims to filter noisy information during the structure generation. In each iteration, we refine the latent structure $S^*$. We employ the scalable randomized SVD~\cite{halko2011finding} to improve the computation efficiency for large-scale graphs. Details are available in Appendix 2.1.
\vspace{-10pt}
\begin{algorithm}
  %\textsl{}\setstretch{1.8}
  \renewcommand{\algorithmicrequire}{\textbf{Input:}}
  \renewcommand{\algorithmicensure}{\textbf{Output:}}
  \caption{The generation algorithm of $S^*$}
  \label{alg1}
  \begin{algorithmic}[1]
 \REQUIRE Coefficient matrix $Q$, subspaces dimension $K=4$, rank $r=Kd+1$, where $d$ is the number of node classes.
 \ENSURE  Latent structure $\mathbf S^*$.
\STATE \textbf{Initialization}: $Q^{\prime}=\frac{1}{2}\left(Q+Q^{T}\right)$.
\STATE \textbf{Compute}: the $r$ rank \textbf{SVD} of $Q^{\prime}$ via $Q^{\prime}=U \Sigma V^{T}$.
\STATE \textbf{Compute}: $L=U \Sigma^{\frac{1}{2}}$ and  normalize each row of $L$.
\STATE \textbf{Update}: $L^{\prime}$ $\leftarrow$ set the negative values in $L$ to zero.
\STATE \textbf{Obtain}: $S^*=\left(L^{\prime}+\right.$ $\left.L^{\prime T}\right) /\|L\|_{\infty}$, where $s_{i j} \in[0,1]$.
\end{algorithmic}\label{algo}
\end{algorithm}

\vspace{-5pt}
\subsubsection{Dual-view Contrastive Learner.}
So far we have obtained the latent structure $S^*$, and in the previous step, we focus on learning $S^*$ based on the node features. To refine such a structure, we further explore the enriched structural information of the graph and propose a novel dual-view contrastive learner. We take four steps for such a refinement.
\noindent
\textbf{Step 1: Generating the dual views of latent structure.}
We denote the graph as $G=(S^*, \mathbf X)$, where $S^*$ is the learnable latent homophilic structure. Based on $G$, we generate two graphs $G_{1}$ and $G_{2}$ via the corruption function~\cite{velickovic2019deep} to refine the structure in a self-supervised manner. Specifically, the corruption function randomly removes a small portion of edges from $S^*$ and also randomly masks a fraction of dimensions with zeros in node features $\mathbf X$.

\noindent
\textbf{Step 2: Aggregating information on latent structure.} \iffalse The above generated latent structure $S^*$ is a probability matrix to depict whether the node pair belongs to the same class.\fi For efficient aggregation on $S^*$, we devise a truncated threshold GCN to control the sparsity of the structure. For $S^*$, we introduce a threshold $\sigma$ to decide if there exists a soft connection with continuous values between two nodes and then form a new structure $S^*_{\sigma}$. Such a way is quite different from the previous hard-coding operations~\cite{liu2022towards} that only have $0$ or $1$, and our $S^*_{\sigma}$ can be flexibly applied to various benchmarks. We employ the truncated threshold GCN on three graphs, including $G$, $G_1$, and $G_2$.
The proposed truncated threshold GCN on graph $G$ can generate the representations as follows:

%we do not employ  on latent structure, but maintain the soft connection between nodes via using the threshold $\sigma$.

%It can enhance the structure learning flexibility when handling different datasets or attacks.
%So we further introduce a tunable threshold $\sigma$ into GCN for controlling structure learning intensity. We denote it as truncated threshold GCN, and as the encoder for next-stage refinement.
%, and the detailed analysis of the threshold for defending attacks are avaliable in Experiment.

%For example, as for small datasets such as Texas, the performances on pure node features-level learning (via MLP) have performed well, so the intensity of structure learning can be reduced by increasing $\sigma$. As for large datasets such as Squirrels and Chameleons, structure learning intensity can be increased by lower $\sigma$ to improve accuracy.

%\iffalseThe $\sigma$ is generally between $1-0.8$.
%propose the truncated threshold GCN to learn node features and new graph structure $S^*$. Specifically, when $S^*$ is obtained, we do not perform hard-coding operations for connected edges, but maintain the soft connection between nodes. Then we control the neighbor aggregation strength through a tunable truncated threshold $\sigma$, and $\sigma$ is generally between 1-0.8. It will further enhance the flexibility of the structure learner for different structural environment.
\vspace{-10pt}
\begin{equation}
\begin{aligned}
S^*&=\{s^*_{ij} \geq \sigma \mid s^*_{ij} \in S^*\} \\
 Z=\operatorname {GCN}(\mathbf X, S^*)&=\hat{S}^* \rm{ReLU}(\hat{S}^* \mathbf X W^{(0)}) W^{(1)}
\end{aligned}
\end{equation}
where $\operatorname{ReLU}$ is an activation function, $W^{(0)}$ and $W^{(1)}$ are the trainable weight matrices, $\tilde{S}^*=S^*+I$, $I \in \mathbb{R}^{|V| \times|V|}$ is the identity matrix and the degree diagonal matrix $\tilde{D}_{i i}$ with $\tilde{D}_{i i}=\sum_{j \in V} \tilde{S}^*_{i j}, \forall i \in V .$ We set $\hat{S}^*=\tilde{D}^{-\frac{1}{2}} \tilde{S}^* \tilde{D}^{-\frac{1}{2}}$. $W^{(0)}$ and $W^{(1)}$ are trainable parameter matrices of GCN.
$Z_{1}$ and $Z_{2}$ denote the node embedding matrices for the two views $G_{1}$ and $G_{2}$, and these node embeddings are generated from the proposed $\mathrm{GCN}$ encoder.

\noindent
\textbf{Step 3: Sampling the contrastive samples.} For a node $v_i \in V$, let us denote the corresponding nodes in $G_{1}$ and $G_{2}$ as $G_{1}(v_i)$ and $G_{2}(v_i)$ respectively. Then we introduce the node-pair sampling rules for the contrastive learning as follows: a) \textbf{ positive example} is the node pair from the same node of different graph views, that is, $\forall i \in V$, the pair $\left(G_{1}(i), G_{2}(i)\right)$. b) \textbf{negative example} is the node pair from the different nodes of the same or different graph views, that is, $\forall i \in V$, $V_{-i}=\{j \in V \mid j \neq i\}$. Both $\left(G_{1}(i), G_{1}(j)\right)$ and $\left(G_{1}(i), G_{2}(j)\right)$ are the negative examples.
\iffalse
\begin{itemize}
    \item Positive example:  $\forall i \in V$, the pair $\left(G_{1}(i), G_{2}(i)\right)$.
    \item Negative example: $\forall i \in V$, $V_{-i}=\{j \in V \mid j \neq i\}$. Both $\left(G_{1}(i), G_{1}(j)\right)$ and $\left(G_{1}(i), G_{2}(j)\right)$ are the negative examples.
    \item Pair-wise example: we sample the node pairs from training set, which have the labels. The same-class node pairs are positive samples, while the different-classes node pairs are negative samples
\end{itemize}
\fi

\noindent
\textbf{Step 4: Optimizing the contrastive loss.} In addition to the above dual-view optimization, we also propose a novel pairwise constraint to optimize the original graph view, which can further improve the quality of the learned homophilic structure. Specifically, we sample the node pairs from the training set with labels. The same-class node pairs are positive samples noted as $(u,v)$, while the different-classes node pairs are negative samples noted as $(u,v_n)$, where $u, v$, and $v_n$ belong to the training node set and $y(u)=y(v)$, $y(u)\neq y(v_n)$. Here $y(\cdot)$ is the node label. We formally propose the loss function as:
\vspace{-6pt}
\begin{equation}
  \begin{aligned}
  & \mathcal{L}_{refine}=\sum_{i \in V}\left[-\frac{\cos \left(z_{1 i}, z_{2 i}\right)}{\tau}\right. \\
  &\left.+\log \left(\sum_{j \in V_{-i}} e^{\frac{\cos \left(z_{1 i}, z_{1 j}\right)}{\tau}}+e^{\frac{\cos \left(z_{1 i}, z_{2 j}\right)}{\tau}}\right)\right] \\
  &-\lambda_2 [\log \left(\sigma\left({z}_u^{\top} {z}_v\right)\right)- \log \left(\sigma\left(-{z}_u^{\top} {z}_{v_n}\right)\right)]
  \end{aligned}\label{dual}
\end{equation}
where $z_{1 i}$ and $z_{2 i}$ denote embeddings node $i$ on $Z_{1}$ and $Z_{2}$ respectively, $z_{(v)}$ denotes embedding of node $v$ on $Z$, $\cos (\cdot)$ is the cosine similarity between the two embeddings, $\tau$ is a temperature parameter, and $\lambda_2$ is a  hyperparameter of the second term. The first term of Eq.~\ref{dual} encourages the consistent information between positive samples, while the second term penalizes the inconsistent information between dual views. The last term of Eq.~\ref{dual} makes sure that the same-class nodes have more similar representations.
\iffalse
This function aims to fully explore the agreement information be

%For each $i \in V$, the pair $\left(G_{1}(i), G_{2}(i)\right)$ is considered as a positive example. Negative examples are sampled from both the views for each node $i \in V$. We randomly select a set of nodes $V_{-i}=\{j \in V \mid j \neq i\}$ such that $\left|V_{-i}\right|=N_{-}$(number of negative samples), $\forall i \in V$. Both $\left(G_{1}(i), G_{1}(j)\right)$ and $\left(G_{1}(i), G_{2}(j)\right)$ are considered as negative examples.
c) pair-wise example, we sample the node pairs from training set, which have the labels. The same-class node pairs are positive samples, while the different-classes node pairs are negative samples.

Additionally, we introduce pairwise constraints to further optimize the original view. (novelty)

It aims to obtain refined embeddings that are more closely to real class constraints. (aim)

%Specifically, we sample the nodes from training set, which have the labels, and the same-class node pairs are positive samples, while the different-classes node pairs are negative samples.
Our goal is to make positive samples embedding more similar, while negative samples embedding remain different.(aim)
Then, the objective function can be expressed as:

 This essentially maximizes the agreement between the embeddings of $i$-th node in two views. \
 \fi
%where the purpose of $G_{1}$ and $G_{2}$ were just to enable the training of the self-supervised loss in Equation 2.
\subsubsection{Structure Refinement.}
The node embedding matrix $\mathbf{Z}$ generated by Eq.~\ref{dual} has incorporated the refined structure and node features. Finally, we feed $\mathbf{Z}$ into the structure inducer again to iteratively refine the structure $S^*$. Equipped with both the original graph $A$ and the refined one $S^*$, we use a structure bootstrapping mechanism $S^{*} \leftarrow \zeta A+(1-\zeta)S^{*}$ to update $S^*$ with a slow-moving of $A$, where $\zeta$ is a hyperparameter to balance the information between $A$ and $S^{*}$. Specifically, the input graphs with high heterophily will lead to smaller $\zeta$, while the ones with high homophily have larger $\zeta$. By doing so, we can reduce the scale of ``right-shift'' over heterophilic graphs and thus potentially mitigate the structural OOD issue under malicious attacks as discussed in Fig.~\ref{fig:intro-attack} in the Introduction Section.

\vspace{-5pt}
\subsection{Graph Encoder}
Our graph encoder consists of a GCN encoder and a GCN decoder, where the former encodes the masked features, and the latter aims to generate the reconstructed features $\hat{\mathbf {X}}$. We feed the masked node features $\mathbf {\tilde {X}}$ and $S^*$ to the graph encoder.  Then we use a scaled cosine error to optimize the encoder as follows:

%\cor{The detailed implementations} of this module are available in appendix.

%al we introduce a remasked graph encoder module to generate node representations for classification task. It emphasizes the importance of feature learning for better classification performance.
%and has been proven to be the powerful node classification model on homophilic graphs. However, it cannot be directly deployed on the heterophilic graphs due to the heterophily, while the structure learner we devised makes it possible.

%\cor{The detailed implementations} of this module are available in appendix. For short, we first use a GCN encoder as $H=\operatorname {GCN}(\mathbf {\tilde {X}}, S)$, where $\tilde {\mathbf X}$ are re-masked features, $S$ is the learned structure. Then we use a GCN decoder $\hat{\mathbf {X}}=\operatorname {GCN}(\mathbf {\tilde {H}}, S)$ to generate the reconstructed features, where $\tilde {H}$ is re-masked encodings. Then we use a scaled cosine error to optimize the encoder as follows:
%First, we give the symbolic representation. GCN represents an encoder or decoder, and here we use a graph convolutional network, which can also be replaced with other backbone networks, such as Graphsage, GAT, etc. [MASK] represents the mask token, and we use the node feature zeroing as the [MASK], the masking rate is $\theta$. $V_{[M]}$ represents the set of nodes sampled from node $V$ at the mask rate $\theta$. Then we formally propose the re-masking feature augmenter.
\vspace{-4pt}
\begin{equation}\label{MAE}
\mathcal{L}_{\mathrm{Re}}=\frac{1}{|\widetilde{\mathcal{V}}|} \sum_{v_{i} \in \widetilde{\mathcal{V}}}\left(1-\frac{\mathbf {x}_{i}^{T} \hat{\mathbf {x}}_{i}}{\left\|\mathbf {x}_{i}\right\| \cdot\left\|\hat{\mathbf {x}}_{i}\right\|}\right)^{\gamma}, \gamma \geq 1
%\vspace{-5pt}
\end{equation}
where $\mathbf x_i$ and $\mathbf {\hat x}_i$ are the feature and reconstructed feature of node $i$, $\gamma$ is a scale factor.
%\cor{Eq.~\ref{MAE} optimizes the reconstruction loss of node features, which has been proven to be more beneficial for node classification tasks.}
\begin{figure*}[t]
  \centering
  \includegraphics[scale=0.395]{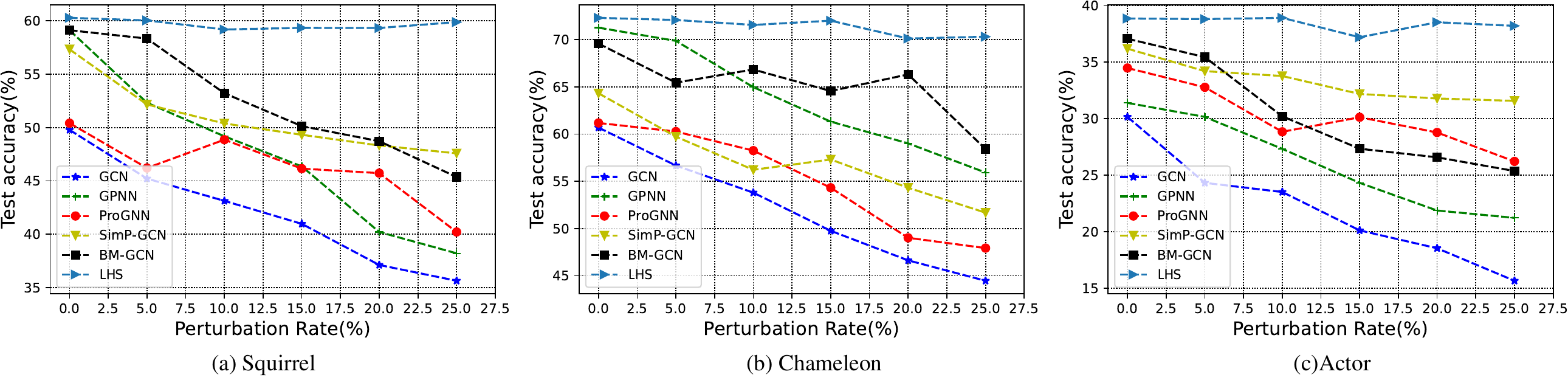}\\
    \vspace{-10pt}
  \caption{Comparisons of node classification under a poisoning attack. We repeat three times and report the mean values.}\label{fig:meta}
\end{figure*}

%0.405
\iffalse
\begin{figure*}[t]
  \begin{minipage}[h]{0.5\linewidth}
  \includegraphics[scale=0.33]{figures/metattack.pdf}
  %\figcaption{variation with XXX}
  \vspace{-15pt}
  \caption{Node classification accuracies under a positioning attack. We repeat three times and report the mean value.}
  \label{fig:layers}
  \end{minipage}%
  \begin{minipage}{0.5\textwidth}
    \flushright
  %\tabcaption{Add caption}
  \begin{tabular}{l}
  \setlength{\tabcolsep}{0.3mm}{
  \small
    \begin{tabular}{cccc}
      \multicolumn{4}{c}{\begin{tabular}[c]{@{}c@{}}Table5: Performance comparisons against \\ \quad \quad \quad \ \ two types of OOD attack\end{tabular}} \\
      \multicolumn{1}{l}{} & \multicolumn{1}{l}{} & \multicolumn{1}{l}{} & \multicolumn{1}{l}{} \\
      \hline
      \textbf{} & Sampled & Customized & \emph{Average} \\ \midrule
      Node2Seq & 50.23 & 40.55 & 45.39 \\
      GPNN & 52.45 & 41.79 & 47.12 \\
      \hline
      RSGCN-in  & 59.78 & 54.27 & 57.03 \\
      RSGCN-tran  & \textbf{68.76} &\textbf{ 62.13} & \textbf{65.45} \\
      \hline
      \multicolumn{1}{l}{} & \multicolumn{1}{l}{} & \multicolumn{1}{l}{} & \multicolumn{1}{l}{} \\
      \multicolumn{1}{l}{} & \multicolumn{1}{l}{} & \multicolumn{1}{l}{} & \multicolumn{1}{l}{} \\
      \multicolumn{1}{l}{} & \multicolumn{1}{l}{} & \multicolumn{1}{l}{} & \multicolumn{1}{l}{}
      \end{tabular}
  }
   \label{tab:addlabel}%
  \end{tabular}
  \end{minipage}
  \end{figure*}
  \fi

\begin{table*}[t]
  \centering
  \vspace{-8pt}
  \caption{Robustness comparisons in terms of classification accuracy(\%) under 2 evasion-based attacks OOD and Injected.}
   \resizebox{\linewidth}{!}{
    \begin{tabular}{c|ccccccccccr}
    \hline
          & \multicolumn{2}{c}{Wisconsin} & \multicolumn{2}{c}{Texas} & \multicolumn{2}{c}{Chameleon} & \multicolumn{2}{c}{Squirrel} & \multicolumn{2}{c}{Actor} &  \\
          & OOD   & Injected & OOD   & Injected & OOD   & Injected & OOD   & Injected & OOD   & Injected & \multicolumn{1}{c}{\textit{average}} \\
    \hline
 H2GCN & $48.24$$\pm$\scriptsize$2.1$ & $41.53$$\pm$\scriptsize$1.7$ & $45.33$$\pm$\scriptsize$0.9$ & $39.86$$\pm$\scriptsize$1.1$ & $47.20$$\pm$\scriptsize$1.5$  & $37.21$$\pm$\scriptsize$2.3$ & $48.27$$\pm$\scriptsize$3.1$ & $36.34$$\pm$\scriptsize$1.6$ & $21.33$$\pm$\scriptsize$2.3$ & $14.17$$\pm$\scriptsize$1.4$ & $37.97$ \\
   GPNN  & $52.78$$\pm$\scriptsize$0.8$ & $40.21$$\pm$\scriptsize$1.4$ & $56.33$$\pm$\scriptsize$1.1$ & $38.78$$\pm$\scriptsize$1.9$ & $54.62$$\pm$\scriptsize$2.0$ & $48.49$$\pm$\scriptsize$1.6$ & $50.23$$\pm$\scriptsize$0.6$ & $40.55$$\pm$\scriptsize$1.3$ & $20.83$$\pm$\scriptsize$2.7$ & $16.28$$\pm$\scriptsize$1.5$ & $41.91$ \\
    \hline
   UGCN & $72.37$$\pm$\scriptsize$2.7$ & $44.58$$\pm$\scriptsize$2.2$ & 50.40$\pm$\scriptsize$1.5$  & $41.92$$\pm$\scriptsize$0.7$ & $57.23$$\pm$\scriptsize$3.1$ & $40.39$$\pm$\scriptsize$2.8$ & $52.45$$\pm$\scriptsize$3.3$ & $41.79$$\pm$\scriptsize$3.2$ & $23.37$$\pm$\scriptsize$1.9$ & $15.57$$\pm$\scriptsize$0.8$ & $44.00$ \\
    SimP-GCN & $73.34$$\pm$\scriptsize$2.1$ & $61.43$$\pm$\scriptsize$2.4$ & $59.76$$\pm$\scriptsize$3.3$ & $60.31$$\pm$\scriptsize$0.4$ & $61.28$$\pm$\scriptsize$1.6$ & $54.57$$\pm$\scriptsize$2.6$ & $54.34$$\pm$\scriptsize$1.0$ & $54.01$$\pm$\scriptsize$2.3$ & $28.96$$\pm$\scriptsize$1.8$ & $24.31$$\pm$\scriptsize$2.9$ & $53.23$ \\
    BM-GCN &  $76.58$$\pm$\scriptsize$0.5$ & $69.78$$\pm$\scriptsize$1.3$ & $62.17$$\pm$\scriptsize$1.1$ & $63.22$$\pm$\scriptsize$1.0$ & $62.37$$\pm$\scriptsize$2.6$ & $52.58$$\pm$\scriptsize$2.3$ & $57.30$$\pm$\scriptsize$1.0$  & $59.82$$\pm$\scriptsize$0.7$ & $26.17$$\pm$\scriptsize$1.4$ & $18.92$$\pm$\scriptsize$1.9$ & $54.89$ \\
    \hline
    LHS & \textbf{82.31}$\pm$\scriptsize$0.5$ & \textbf{73.34}$\pm$\scriptsize$1.1$ & \textbf{71.15}$\pm$\scriptsize$1.4$ & \textbf{68.19}$\pm$\scriptsize$1.6$ & \textbf{67.33}$\pm$\scriptsize$0.9$ & \textbf{60.10}$\pm$\scriptsize$2.2$  & \textbf{68.76}$\pm$\scriptsize$3.1$ & \textbf{62.13}$\pm$\scriptsize$2.7$ & \textbf{32.23}$\pm$\scriptsize$2.6$&
    \textbf{33.42}$\pm$\scriptsize$1.8$ & \textbf{62.41} \\
    %LHS-in & $79.23$$\pm$\scriptsize$1.1$ & $67.26$$\pm$\scriptsize$0.7$ & $68.83$$\pm$\scriptsize$1.4$ & $62.72$$\pm$\scriptsize$2.2$ & $63.17$$\pm$\scriptsize$1.6$ & $56.32$$\pm$\scriptsize$1.5$ & $59.78$$\pm$\scriptsize$0.9$ & $60.27$$\pm$\scriptsize$1.3$ & $30.24$$\pm$\scriptsize$1.4$ & $27.71$$\pm$\scriptsize$1.1$ & $57.55$ \\
   \hline
    \end{tabular}}%
  \label{tab:exp-ood-injected}%
  \vspace{-5pt}
\end{table*}%

\subsection{Classifier and Loss Functions}
Finally, our classifier outputs predictions. We generate classification representations via a fully-connected layer $F(\cdot)$, that is $y_{\text {pred}}=\operatorname{softmax}(F (\mathbf h_i ))$.
Then the loss of the classifier can be expressed as $\mathcal L_{Pre}=\sum_{i=1}^{N_{l}} y_{i} \log y_{\text {pred } i} $. We jointly train the graph encoder and classifier with $\mathcal L$, which can be expressed as:
\iffalse
\begin{equation}
y_{\text {pred }}=\operatorname{softmax}\left(F \left(z_{\text {fi }}\right)\right)
\end{equation}
\fi
\vspace{-4pt}
\begin{equation}
\mathcal{L}=\mathcal{L}_{Pre}+\beta \mathcal{L}_{Re}
\end{equation}
where $\beta$ is a hyperparameter of loss weight.

\section{Experiments}
\subsection{Datasets, Baselines and Settings}
%\vspace{-2pt}
\paragraph{Datasets:} We experiment on nine benchmarks. For six heterographic spatial datasets including \textbf{Cornell}, \textbf{Texas}, \textbf{Wisconsin}~\cite{Pei2020Geom-GCN:}, \textbf{Chameleon}, \textbf{Squirrel}~\cite{rozemberczki2021multi}, nodes are web pages and edges are hyperlinks between these pages, and \textbf{Actor}~\cite{tang2009social}, nodes are actors and edges denote co-occurrences on same web pages. For the three homophilic datasets including  \textbf{Cora}, \textbf{Citeseer} and \textbf{PubMed}~\cite{yang2016revisiting}, nodes refer to articles, and edges are the citations between articles. Due to space limitations, we provide detailed descriptions in Appendix 4.1.

%$\bullet$  are two heterophilic Wikipedia pages datasets, where nodes are Wikipedia pages, and each edge represents the hyperlink between pages.

%$\bullet$} are three homophilic citation datasets. The different node of them represents the article belonging to different research areas, and

%And the detailed and statistical information about these datasets are shown in Supplementary.
%\vspace{-5pt}
\paragraph{Baselines:}
We follow the previous works~\cite{jin2021node,he2022block} to use eleven baselines. We categorize these methods into three groups:
\textbf{1) multi-hop-based approaches} MixHop~\cite{abu2019mixhop} and H2GCN~\cite{zhu2020beyond}, which mix the multi-hop neighbors for aggregation; \textbf{2) ranking-based approaches} NLGNN~\cite{liu2021non}, GEOM-GCN~\cite{Pei2020Geom-GCN:}, Node2Seq~\cite{yuan2021node2seq} and GPNN~\cite{yang2022graph} that aim to search on the network structure and then perform selective aggregation; \textbf{3) structure learning approaches} ProGNN~\cite{jin2020graph}, UGCN~\cite{jin2021universal}, BM-GCN~\cite{he2022block} and GREET~\cite{liu2023beyond} that automatically learn graph structures for aggregations. Specifically, \textbf{ProGNN} preserves the low-rank and sparsity characteristics of the graph structure for robust GCN. \textbf{UGCN} and \textbf{SimP-GCN} employ a similarity preservation scheme for structure learning on heterophilic graphs and \textbf{BM-GCN} employs a selective aggregation on structure via a block-guided strategy. We also compare our model with a recently proposed spectral-based method ALT-GCN~\cite{xu2023node}.
%Note that GPNN and BM-GCN are state-of-the-art methods for node classification and structure learning respectively on heterophilic graphs.

%Additionally, we also use \textbf{MLP} ~\cite{yang2022graph} that only relies on node features for classification.

% \textbf{ProGNN}~\cite{jin2020graph} that preserves the low-rank and sparsity characteristics of the graph structure for robust GCN. \textbf{UGCN}~\cite{jin2021universal} and \textbf{SimP-GCN}~\cite{jin2021node} that employ a similarity preservation scheme for structure learning on heterophilic graphs and \textbf{BM-GCN}~\cite{he2022block} that employs a selective aggregation on structure via a block-guided strategy.
%\vspace{-5pt}

\paragraph{Settings:} We implement our method by Pytorch and Pytorch Geometric and use Adam Optimizer on all datasets with the learning rate as $0.001$. We configure epochs as $1000$ and apply early stopping with the patience of $40$. We configure the hidden size as $64$ and the batch size as $256$. We perform the structure learning for $2$ rounds. More detailed hyperparameters are available in Appendix 4.3.

%For the structure inducer, we configure corruption rates of $G_1$ and $G_2$ $\in [0.2,0.4]$. We follow~\cite{zhu2020deep} to set self-expressive parameter $\lambda_1$ and temperature parameter $\tau$ as $0.7$ and $0.6$, respectively. The step size of the pairwise constraint $\lambda_2 \in [0,2]$. The truncated threshold $\sigma \in [0.8,1]$ is traversed with a step size of 0.05 and decay rate $\zeta=0.9999$.For the graph encoder, we follow GraphMAE~\cite{hou2022graphmae} to configure the scaled factor $\gamma \in [1,3]$. The step size of the re-mask ratio $\theta \in [0.5,0.8]$. $\beta \in [0,2]$ is set in the classifier with a stepsize of $0.1$.

%\cor{todo by chenyang: for appendix, try to add section number for appendix, such as appendix B.2}
%The edges of graph are destroyed with a probability of 0 or 0.2 and the features of it are destroyed with a probability of 0.2 or 0.3; set the pairwise constraint of semi-supervised learning to 2 and use the same values of $\lambda$ and $\tau$ in the literature~\cite{zhu2020deep}; set the truncation threshold $\in \{0.8, 1\}$ and use 2-layer GCN with $50 \%$ or $75 \%$ as probability to mask or re-mask features. Epochs was run $\in \{300, 1000\}$ with learning rate $\in \{0.000005, 0.0001\}$ and the hidden unit $\in \{16, 32\}$, dropout in each layer $\in \{0.1, 0.2\}$, weight decay of $5e^{-4}$.\end{small}
\vspace{-3pt}
\subsection{Main Results}
\label{sec:robust-exp}
%We compare the robustness of the proposed LHS model with baseline methods under various attacks.

%Previous studies on GCN robustness primarily focused on homophilic graphs, but little research has been done on heterophilic graph structures, which are more complex and pose a greater risk. Therefore, it is crucial to conduct robustness studies against structural attacks on heterophilic graphs.
%\vspace{-8pt}

\begin{table*}[t]
    \centering
    \small
    \caption{Comparisons of node classification without any attacks.}
    \vspace{-5pt}
    \resizebox{\linewidth}{!}{
     \begin{tabular}{c|cccccccccc}
      \hline
            & Wisconsin & Texas & Cornell & Chameleon & Squirrel & Actor & Cora  & Citeseer & {Pubmed} & \textit{Average} \\
      \hline
      MixHop & $75.88$$\pm$\scriptsize$4.9$ & $77.84$$\pm$\scriptsize$7.7$ & $73.51$$\pm$\scriptsize$6.2$ & $60.50$$\pm$\scriptsize$2.5$ & $43.80$$\pm$\scriptsize$1.4$ & $32.22$$\pm$\scriptsize$2.3$ & $81.90$$\pm$\scriptsize$0.8$ & $71.40$$\pm$\scriptsize$1.3$ & $80.80$$\pm$\scriptsize$0.6$ & $66.43$ \\
      H2GCN & $86.67$$\pm$\scriptsize$4.6$ & $84.86$$\pm$\scriptsize$6.7$ & $82.16$$\pm$\scriptsize$6.0$ & $57.11$$\pm$\scriptsize$1.6$ & $37.90$$\pm$\scriptsize$2.0$ & $35.86$$\pm$\scriptsize$1.0$ & $87.81$$\pm$\scriptsize$1.3$ & $77.07$$\pm$\scriptsize$1.6$ & $89.59$$\pm$\scriptsize$0.3$ & $71.00$ \\
      \hline
      NLGNN & $87.30$$\pm$\scriptsize$4.3$ & $85.40$$\pm$\scriptsize$3.8$ & $84.90$$\pm$\scriptsize$5.7$ & $70.10$$\pm$\scriptsize$2.9$ & $59.00$$\pm$\scriptsize$1.2$ & $37.90$$\pm$\scriptsize$1.3$ & $88.50$$\pm$\scriptsize$1.8$ & $76.20$$\pm$\scriptsize$1.6$ & $88.20$$\pm$\scriptsize$0.5$ & $75.28$ \\
      GEOM-GCN & $65.10$$\pm$\scriptsize$6.5$ & $67.84$$\pm$\scriptsize$5.8$ & $60.00$$\pm$\scriptsize$6.5$ & $65.81$$\pm$\scriptsize$1.6$ & $45.49$$\pm$\scriptsize$1.3$ & $31.94$$\pm$\scriptsize$1.0$ & $85.65$$\pm$\scriptsize$1.7$ & \textbf{79.41}$\pm$\scriptsize$1.7$ & \textbf{90.49}$\pm$\scriptsize$0.3$ & $65.75$ \\
      Node2Seq & $60.30$$\pm$\scriptsize$7.0$ & $63.70$$\pm$\scriptsize$6.1$ & $58.70$$\pm$\scriptsize$6.8$ & $69.40$$\pm$\scriptsize$1.6$ & $58.80$$\pm$\scriptsize$1.4$ & $31.40$$\pm$\scriptsize$1.0$ & -     & -     & -     & $57.05$ \\
      GPNN  & $86.86$$\pm$\scriptsize$2.6$ & $85.23$$\pm$\scriptsize$6.4$ & $85.14$$\pm$\scriptsize$6.0$ & $71.27$$\pm$\scriptsize$1.8$ & $59.11$$\pm$\scriptsize$1.3$ & $37.08$$\pm$\scriptsize$1.4$ & -     & -     & -     & $70.79$ \\
      \hline
      UGCN  & $69.89$$\pm$\scriptsize$5.2$ & $71.72$$\pm$\scriptsize$6.2$ & $69.77$$\pm$\scriptsize$6.7$ & $54.07$$\pm$\scriptsize$1.7$ & $34.39$$\pm$\scriptsize$1.9$ & -     & $84.00$$\pm$\scriptsize$0.9$ & $74.08$$\pm$\scriptsize$1.2$ & $85.22$$\pm$\scriptsize$0.7$ & $67.89$ \\
      SimP-GCN & $85.49$$\pm$\scriptsize$3.5$ & $81.62$$\pm$\scriptsize$6.5$ & $84.05$$\pm$\scriptsize$5.3$ & -     & -     & $36.20$$\pm$\scriptsize$1.3$ & -     & -     & -     & $71.84$ \\
      BM-GCN & - & $85.13$$\pm$\scriptsize$4.6$ & -     & $69.58$$\pm$\scriptsize$2.9$ & $51.41$$\pm$\scriptsize$1.1$ & -     & $87.99$$\pm$\scriptsize$1.2$ & $76.13$$\pm$\scriptsize$1.9$ & $90.25$$\pm$\scriptsize$0.7$ & - \\
      GREET & $84.90$$\pm$\scriptsize$3.3$ & \textbf{87.00}$\pm$\scriptsize$4.2$ & $85.10$$\pm$\scriptsize$4.9$ & $63.60$$\pm$\scriptsize$1.2$ & $42.30$$\pm$\scriptsize$1.3$ & $36.60$$\pm$\scriptsize$1.2$ & $83.81$$\pm$\scriptsize$0.9$ & $73.08$$\pm$\scriptsize$0.8$ & $80.29$$\pm$\scriptsize$1.0$ & $70.74$ \\
     ALT-GCN & $76.40$$\pm$\scriptsize$3.9$ & $70.90$$\pm$\scriptsize$4.3$ & $73.90$$\pm$\scriptsize$5.1$ & $65.80$$\pm$\scriptsize$0.9$ & $52.40$$\pm$\scriptsize$0.8$ & -     & $81.20$$\pm$\scriptsize$0.5$ & $71.40$$\pm$\scriptsize$0.4$ & $79.10$$\pm$\scriptsize$0.9$ & $71.39$ \\
      \hline
      \textbf{LHS} &\textbf{ 88.32}$\pm$\scriptsize$2.3$ & $86.32$$\pm$\scriptsize$4.5$ & \textbf{85.96}$\pm$\scriptsize$5.1$ & \textbf{72.31}$\pm$\scriptsize$1.6$ & \textbf{60.27}$\pm$\scriptsize$1.2$ & \textbf{38.87}$\pm$\scriptsize$1.0$ & \textbf{88.71}$\pm$\scriptsize$0.7$ & $78.53$$\pm$\scriptsize$1.5$ & $87.65$$\pm$\scriptsize$0.3$ & \textbf{76.34} \\
      \hline
     \end{tabular}}
  \label{tab:accuracy}%
\vspace{-4mm}
\end{table*}%

\iffalse
    \hline
        MLP   & $85.29$$\pm$\scriptsize$3.6$ & $81.08$$\pm$\scriptsize$4.7$ & $82.16$$\pm$\scriptsize$6.3$ & $47.36$$\pm$\scriptsize$2.5$ & $29.82$$\pm$\scriptsize$1.8$ & $35.79$$\pm$\scriptsize$0.9$ & $75.69$$\pm$\scriptsize$2.2$ & {$74.02$$\pm$\scriptsize$2.1$} & $87.16$$\pm$\scriptsize$0.3$ & $66.51$ \\
              HOG-GCN & $86.67$$\pm$\scriptsize$3.3$ & $85.17$$\pm$\scriptsize$4.4$ & $84.32$$\pm$\scriptsize$4.3$ & -     & -     & -     & $87.04$$\pm$\scriptsize$1.1$ & $76.15$$\pm$\scriptsize$1.8$ & $88.79$$\pm$\scriptsize$0.4$ & - \\

\fi

%\cor{todo by zyteng: might be discussed in the introduction}
\subsubsection{Comparisons under Poisoning Attacks.}
%Another type of structural out-of-distribution attack is the \textbf{poisoning attack}, which aims to harm model accuracy by altering the graph structure.
We compare the robustness of our LHS with five baseline approaches under a popular poisoning attack ~\cite{jin2020graph} on three benchmarks including Squirrel, Chameleon, and Actor.
Under various perturbation rates ranging from $0$ to $25\%$, Figure \ref{fig:meta} shows that our LHS consistently performs best among all baselines. For example, ours yields higher classification accuracy of up to $20$ points compared to ProGNN. These results confirm the superiority of our latent structure learning scheme against poisoning attacks. The existing structure learning methods, including BM-GCN, SimP-GCN, and ProGNN, are also extremely vulnerable under large positioning perturbation rates. Nevertheless, they are better than the other two, showing the promise of structure learning over heterophilic graphs. We also observe that the positioning perturbations, which can significantly degrade the baselines at a large rate (i.e., $25\%$), have a very slight impact on our method. We attribute such gains to the latent structure that can be resistant to the structural OOD issue discussed in the Introduction Section, which will also be illustrated in the first question of the Discussion Section.

%More comparison results are available in Appendix \cor{B}.

%it is nearly immune to the

%As the perturbation rate increases, LHS demonstrates strong robustness and minimal loss of accuracy. This is because LHS discards the original heterophilic structure and generates a robust structure. It's worth noting that this generated structural distribution is sufficient for achieving competitive accuracy results. In contrast, other methods exhibit worse robustness after network structure perturbation because they rely on the original structure.

\subsubsection{Comparisons under Evasion Attacks.}
\begin{figure}[t]
  \centering
  \includegraphics[width=5.8cm]{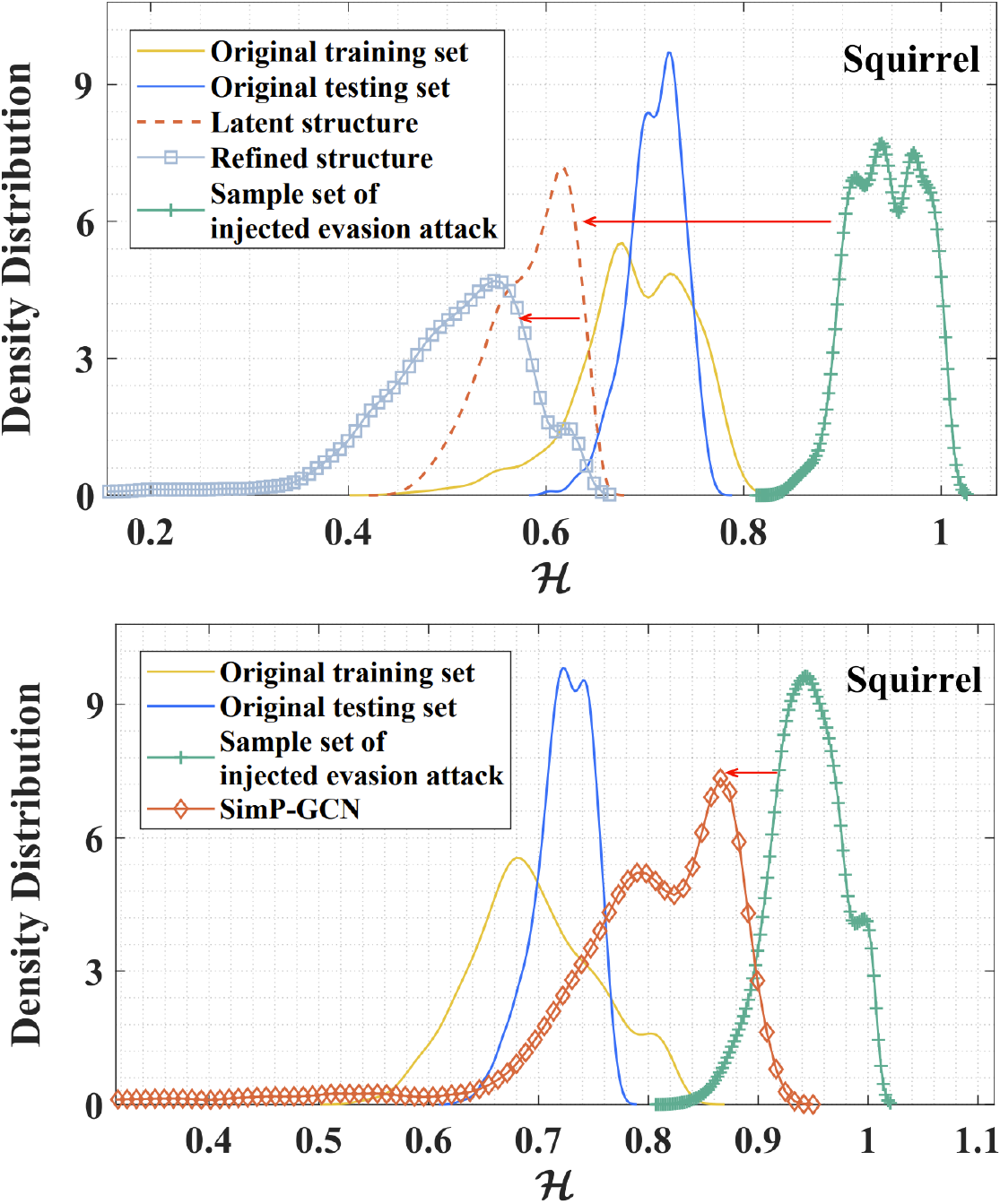}\\
  \vspace{-4mm}
  \caption{Comparisons of the ``right-shift'' of $\mathcal H_{E}$. }\label{fig:dis-shift}
  \vspace{-6mm}
\end{figure}

We presented two evasion-based attacks \cite{Adversarial_Feature_Selection_against_Evasion_Attacks}, i.e., ``OOD evasion attack (OOD)'' and ``injected evasion attack (Injected)'' in Fig.~\ref{fig:intro-attack} (b) and Fig.~\ref{fig:intro-attack}(c), to craft attack samples with destructive structural perturbations to the edges of the graph. Here we compare our method with five baselines on five heterophilic graphs, and report the results in Table \ref{tab:exp-ood-injected}. We chose these five baselines because they are representative of different types of GCN and have been widely used in previous studies. For two nodes with different classes, the ``Injected'' attacks manipulate to inject a connection with a $0.9$ probability.
We repeat our experiments three times and report the mean and variance values in Table \ref{tab:exp-ood-injected}. Under the two attacks, Table \ref{tab:exp-ood-injected} shows that our method consistently achieves the best among all baselines on five benchmarks. Compared with ``OOD'' attacks, ``Injected'' attacks are much more constructive as they significantly increase of heterophily of the testing set. Compared to the state-of-the-art structure learning method BM-GCN, our LHS achieves an $11.33$ point accuracy under ``OOD'' attacks. Overall, these results suggest that LHS is more robust against both attacks compared to all the considered baselines. The key to these improved results is the ability of our LHS to perform global searches of the homophilic structure learned by the structure inducer. %More details can be found \cor{Appendix X}.
\begin{figure*}[t]

  \centering
  \subfigure[visualization of a local structure on squirrel]{
  \label{fig:vis:ratio}
  \begin{minipage}[t]{0.4\linewidth}
  \centering
  \includegraphics[scale=0.18]{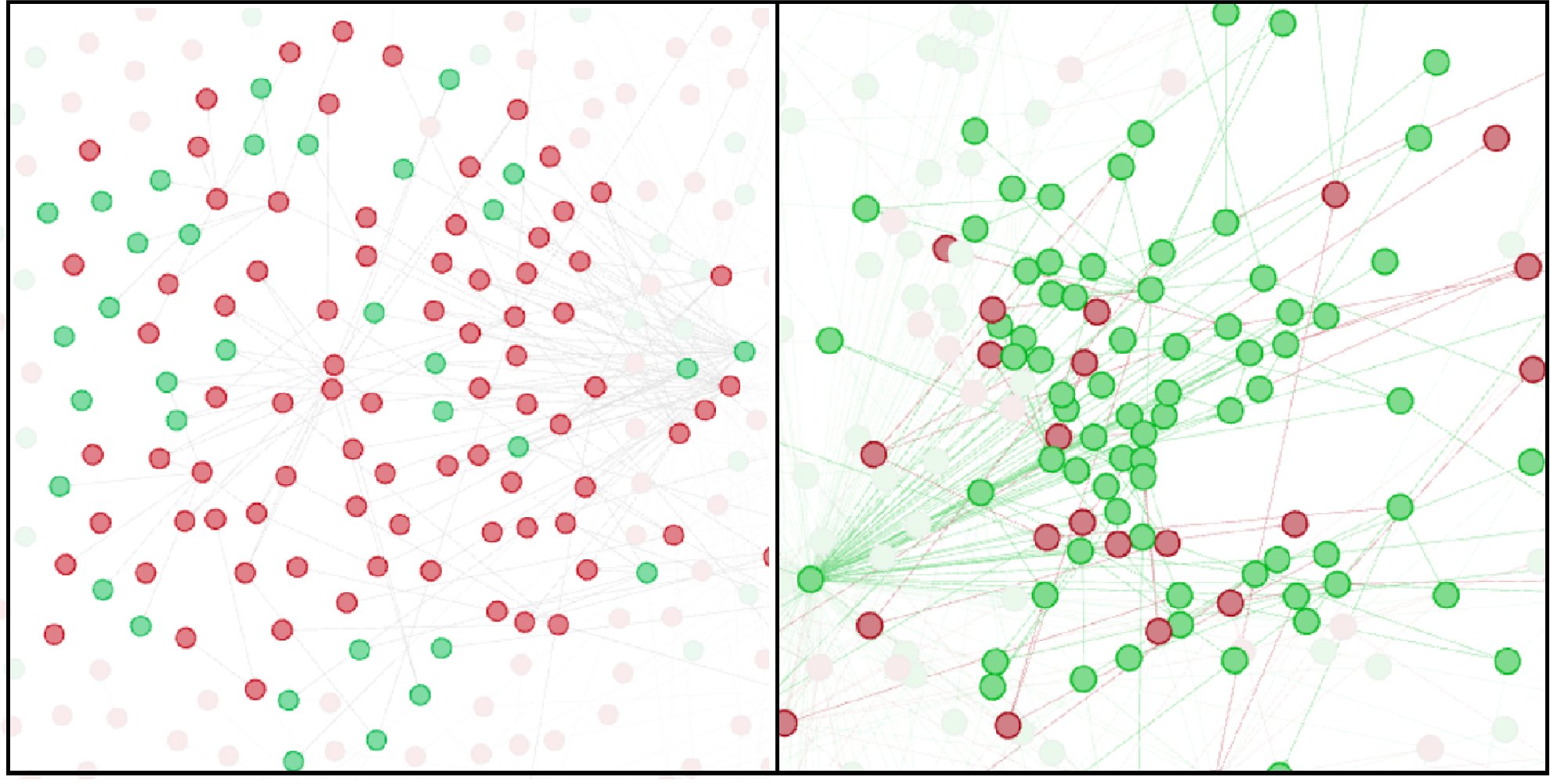}
  %\caption{fig1}
  \end{minipage}%
  }%
  \subfigure[visualization of a node connection on squirrel]{
  \label{fig:vis:case}
  \begin{minipage}[t]{0.4\linewidth}
  \centering
  \includegraphics[scale=0.11]{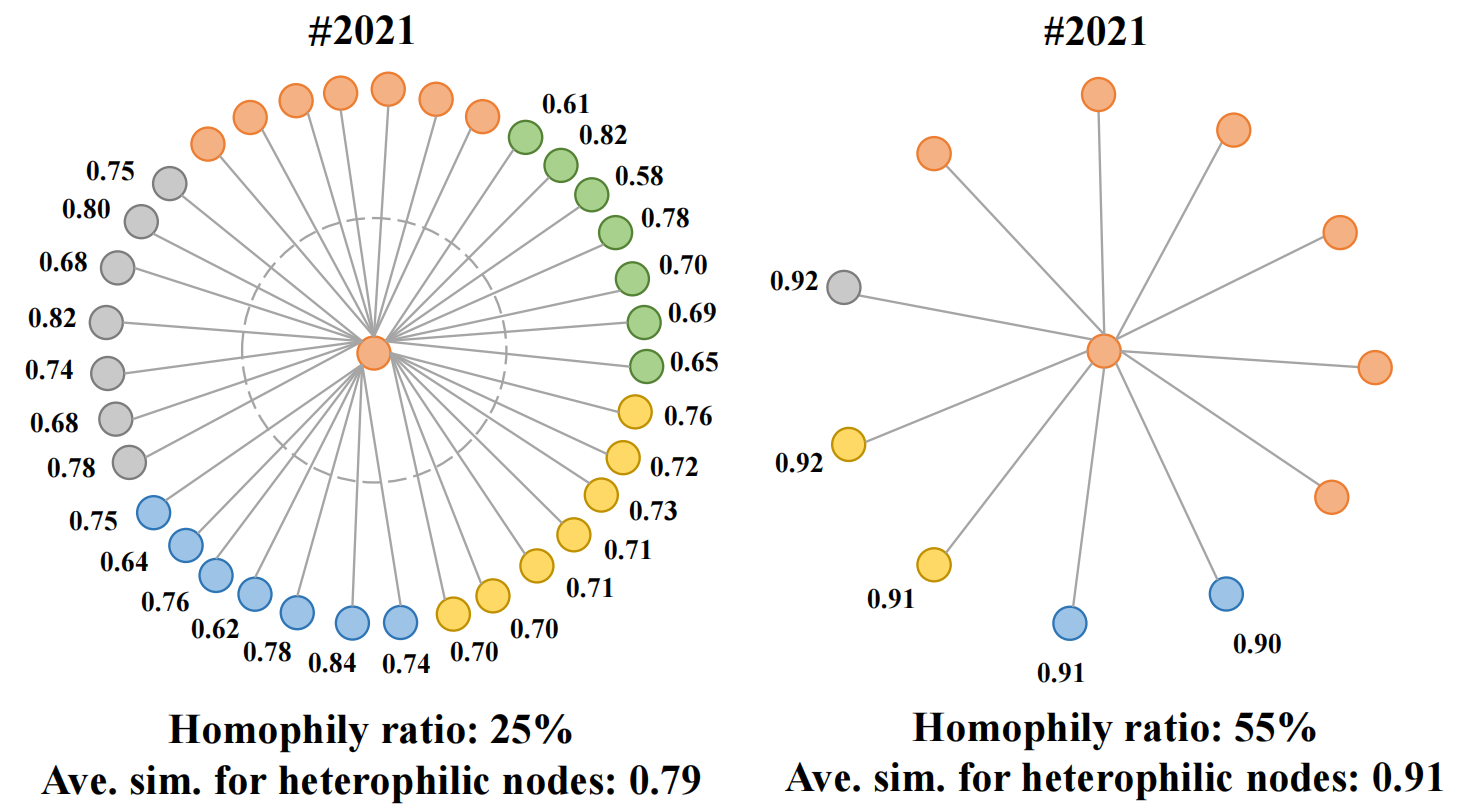}
  %\caption{fig2}
  \end{minipage}%
  }%
  \vspace{-10pt}
  \caption{The visualization of
the latent structure learned by the proposed LHS.}
  \iffalse
  \subfigure[Parameter analysis]{
  \label{fig:vis:param}
  \begin{minipage}[t]{0.35\linewidth}
  \centering
  \includegraphics[width=1.9in]{figures/heat.pdf}
  %\caption{fig2}
  \end{minipage}
  }%
  \centering
  \vspace{-10pt}
  \caption{The visualization of structure analysis and parameter analysis.}
  \label{fig:vis}
  \fi
  \vspace{-20pt}
  \end{figure*}

\vspace{-5pt}
\subsubsection{Comparisons without attacks.}
We have shown that our model is more robust than existing methods under various attacks. To further investigate the performance without attacks, we conduct experiments on the five heterophilic graphs and compare ours with baseline approaches. Table \ref{tab:accuracy} shows that the proposed LHS performs best and we attribute this to the information aggregation in a homophilic way on heterophilic graphs. Additionally, we also achieve better or comparable classification results on three benchmarks for homophilic graphs. This suggests that improving homophily for both homophilic and heterophilic graphs benefits node classification, and this also remotely aligns with a previous work~\cite{yan2021two}, showing that our method can handle both types of graphs in a unified manner.

\subsection{Discussion}
%We raise several questions here to further analyze the proposed method.
\label{sec:discussion}
\iffalse
\subsubsection{Performance on homophilic graphs}
We analyze the interesting observations on the performance of homophilic graphs in Table~\ref{tab:accuracy} as follows: \textbf{1) The performances of homophilic graphs show that the structure refinement of LHS is effective and Competitive for both heterophilic and homophilic graphs.} 2) We find that LHS achieves state-of-the-art performance on Cora, which may be attributed to the ability of LHS to refine heterophilic structures. Because compared with other homophilic graphs, Cora has more heterophilic structure, as show in the misclassified nodes analysis in \cor{appendix}.
\cor{Too short. Table 2 should be highlighted and emphrasize that you can do all the task using a unified view, which is well corresponding to the claim. When you are writing the results, first recall what you claims in the introduction are and then the experiments are supporting evidences to your claims.}
\fi
\noindent\textbf{Can LHS reduce the scale of ``right-shift'' of $\mathcal{H}$ distributions?}
\label{sec:right-shift-ood}
We have discussed that the ``right-shift'' phenomenon, i.e., the structural OOD, is the cause of performance degradation under attacks in the Introduction Section. To answer this question, we visualize how our method reduces the ``right-shift'' for experiments in Table \ref{fig:intro-attack} on Squirrel. Under the ``injected evasion attack'', Figure \ref{fig:dis-shift} shows that our latent structure can greatly move the $\mathcal{H}$ distribution of the attacking sample to the left, thus reducing the ``right-shift'' (see the red arrow). We also observe that the second round of refinement can further move distribution to the left side, further improving the model's robustness. However, we find that existing SimP-GCN can slightly move the distribution, visually explaining why LHS is more robust than Simp-GCN. This further confirms our hypothesis that reducing ``right-shift'' can harden the GCN over heterophilic graphs.

\noindent\textbf{What does the learnable homophilic structure look like?}
For this question, we rely on a tool Gephi~\cite{bastian2009gephi} to visualize the structure of the test set of Squirrel and the homophilic structure of LHS, with a truncated threshold $\sigma$ set as $0.91$. Figure \ref{fig:vis:ratio} demonstrates the visualizations of the two structures respectively. We can observe that there are many more homophilic edges in our structure, compared with the one of the original test set. We also study the connection of the node $\#$2021 of Squirrel. Figure \ref{fig:vis:case} shows that the homophily ratio, which is computed by $1-\mathcal{H}v_{2021}$, is significantly increased from $25\%$ to $55\%$ after the structure latent learning. In this case, we observe that our model reduces the heterophilic connections, while keeping the homophilic edges unchanged.

%Figure \ref{fig:vis:ratio} generate a graph based on our learnable latent structure, which is a fully-connected attention matrix.

\iffalse
\noindent\textbf{How two hyparameters $\sigma$ and $\beta$ affect the node classification?}
The parameter $\sigma$ and $\beta$ indicate the truncated threshold for pruning the structure and the weight of encoder loss for node representation learning, respectively. For this question, we visualize the effect of the two parameters on the classification accuracy in Figure \ref{fig:vis:param}. It shows that both two parameters have an impact on the performance, while the impact of $\sigma$ is larger as it can lead to very low accuracy (e.g., 40\% with $\sigma = 0.98$).   Configuring $\sigma$ near $0.92$ and $\beta$ near $1.2$ can achieve best classification accuracy. More detailed analyses are available in Appendix.
\fi

\noindent\textbf{Can the learnable homophilic structure be applied to other tasks?}
To answer this question, we also apply the homophilic structure learned on four graphs, including Wisconsin, Squirrel, Chameleon, and Cora, to the graph clustering task. We use vanilla GCN~\cite{kipf2017semisupervised} and the proposed structure inducer of LHS to develop ``GCN + Structure Inducer''. Even on the vanilla GCN, Table~\ref{tab:cluster} shows that our ``GCN + Structure Inducer'' outperforms all other baselines on heterophilic
graphs. For example, ours outperforms the Simp-GCN on Squirrel by $2.79$ points. We attribute such again to our homophilic structure.

\begin{table}[h]
  \centering
  \vspace{-7pt}
  \caption{Performance Comparisons on graph clustering}
  \vspace{-8pt}
  \resizebox{\linewidth}{!}{
    \begin{tabular}{c|ccc}
    \hline
          & Wisconsin & Squirrel & Chameleon  \\
    \hline
    SimP-GCN & $58.42$ & $38.57$ & $46.44$  \\
    BMGCN & $54.92$ & $40.26$ & $50.17$  \\
    AGC   & $43.71$ & $32.98$ & $$35.78$$  \\
     GCN + Structure Inducer &  \textbf{61.32} & \textbf{41.36} & \textbf{52.37}  \\
     \hline
    \end{tabular}}%
  \label{tab:cluster}%
\end{table}%

\vspace{-7pt}

\section{Conclusion}
\iffalse
This paper first studies the structural robustness problems of GCNs over heterophilic graphs. From the quantitative view, we endeavor to formulate the structural OOD problems and have found the distribution `right shift' evidence on heterophilic structure. To mitigate the OOD risks, we propose heterophilic structure learning to hardern GCN models robustness. To achieve the strategy, we introduce class-constrained self-expressive technique into structure learning for the first time, while using dual-view contrastive learning to refine the learned structure. Please note that both homophilic and heterophilic graphs can benefit from this structure learning due to our unified perspective. In the future, more efficient models can be proposed to approximate our proposed structure learning strategy; and based on proposed structural distribution formulation tool, the causal invariant substructure learning on the original structure can also be explored.
\fi
This paper studies robust graph convolution networks over heterophilic graphs. We take the first step towards quantitatively analyzing the robustness of GCN approaches over omnipresent heterophilic graphs for node classification, and reveal that the vulnerability is mainly caused by the structural out-of-distribution (OOD). Based on this crucial observation, we present LHS, a novel method that aims to harden GCN against various attacks by learning latent homophilic structures on heterophilic graphs. Our LHS can iteratively refine the latent structure during the learning process, facilitating the model to aggregate information in a homophilic way on heterophilic graphs. Extensive experiments on various benchmarks show the effectiveness of our approach. \iffalse Besides the graph clustering as discussed in Section \ref{sec:discussion},\fi We believe our structure can also benefit more graph tasks for better representation learning. Future work could focus on the development of novel adversarial training methods based on the structural OOD.

%From a quantitative perspective,  we aim to formulate the problem of structural OOD on heterophilic graphs and \cor{have found evidence of a 'left shift' distribution in these structures} \cor{todo by chenyang: verify this sentence}. To address the OOD risks, we propose a heterophilic structure learning method to improve the robustness of GCN models. To implement this strategy, we introduce class-constrained self-expressive techniques for structure learning for the first time, and use dual-view contrastive learning to refine the learned structure. It's worth noting that our method can benefit both homophilic and heterophilic graphs due to our unified perspective. In the future, more efficient models can be proposed to approximate our proposed structure learning strategy.  Based on our proposed structural distribution formulation tool, causal invariant substructure learning on original structure can also be explored.
%\cor{zyteng is here}

\section{Acknowledgments}
This work was partially supported by the National Key R\&D Program of China (Grant No.2022YFB2902200), Major Projects of National Natural Science Foundation of China (Grant  No.72293583), and the Joint Funds for Regional Innovation and Development of the National Natural Science Foundation of China (No. U21A20449).

\bibliography{aaai24}
\appendix
\onecolumn

\section{\textbf{Appendix}}

\section{1. Structure Learning and ``Right-Shift" Phenomenon}
%首先，分布右移，即结构OOD，是因为邻居异质性节点增多，那么“避免右移”则可以使得邻居同质性节点增多，这是容易理解的。本节进一步通过下述频域定理来阐明“避免右移”对于设计LHS以增强节点分类鲁棒性的必要性。
The ``right-shift" of distribution, namely structural OOD, is caused by the high heterophilic neighbor structure, and learning the homophilic neighbor structure can be seen as one of the direct ways to "refrain from right-shift".
In this section, we provide the following analysis from a spectral perspective, aiming to bridge the gap between ``refrain from right-shift" and structure learning, thus elaborating the rationale of LHS to improve classification robustness.

%This section will further clarify the necessity of "avoiding right-shift" for designing LHS to enhance the robustness of node classification through the following theoretical analysis from a spectral perspective.

\subsection{1.1. Theoretical Analysis}
Given a graph signal $\mathbf x$ defined on a graph $G$ with normalized Laplacian matrix $L=I-\tilde{D}^{-\frac{1}{2}} \tilde{A} \tilde{D}^{-\frac{1}{2}}$, and the laplacian regularization term can be proposed as follows:
% the structure learning is actually the laplacian regularization term,
\begin{equation}
\mathbf{X}^{T} L \mathbf{ X}=\frac{1}{2} \sum_{i, j} \tilde{\mathbf{A}}_{i j}\left(\frac{\mathbf{x}_{i}}{\sqrt{1+d_{i}}}-\frac{\mathbf{x}_{j}}{\sqrt{1+d_{j}}}\right)^{2} .
\end{equation}\label{lap}
where $d_{i}$ and $d_{j}$ denotes the degree of nodes $v_{i}$ and $v_{j}$ respectively. Thus a smaller value of $\mathbf{x}^{T} L \mathbf{x}$ indicates a more similar graph signal, i.e., a smaller signal difference between adjacent nodes. This laplacian regularization can also be seen as a smooth or average between two nodes signals (features).
Specifically, from the perspective of graph signals optimization, we will show the relationship between GCN structure learning and the ``right-shift" phenomenon.
\newtheorem{theorem}{Theorem}
\begin{theorem}
From the perspective of graph signal optimization, the laplacian regularization of Eq.~\ref{lap}, is equivalent to the GCN convolutional operator, which is responsible for structure learning.
\end{theorem}

\begin{proof}
First, the signal optimization with a laplacian regularization term can be formally proposed as follows:
\begin{equation}
\arg \min _{\mathbf{Z}} g(\mathbf{Z})=\left\|\mathbf{Z}-\mathbf{X}\right\|^{2}+\lambda \mathbf{X}^{T} L \mathbf{ X},
\end{equation}\label{opGCN}
where $\mathbf Z$ is the learned node embeddings, which can be used for downstream tasks such as node classification.

Then the gradient of $g(\mathbf{Z})$ at $\forall \mathbf{x} \in \mathbf X$ is calculated as:
\begin{equation}
\nabla g\left(\mathbf{x}\right)=2\left(\mathbf x-\mathbf x\right)+2 \lambda L\mathbf{x}_{0}=2 \lambda  L\mathbf{x}_{0} .
\end{equation}
And the one step gradient descent at $\forall \mathbf{x} \in \mathbf X$ with learning rate 1 is formulated as:
\begin{equation}
\begin{aligned}
\mathbf x-\nabla g\left(\mathbf x\right) &=\mathbf x-2 \lambda L \mathbf x \\
&=(\mathbf{I}-2 \lambda L )\mathbf x \\
&=\left(\tilde{{D}}^{-\frac{1}{2}} \tilde{{A}} \tilde{{D}}^{-\frac{1}{2}}+L-2 \lambda L \right) \mathbf x .
\end{aligned}
\end{equation}
By setting $\lambda$ to $\frac{1}{2}$, we finally arrives at:
\begin{equation}
\mathbf x-\nabla g\left(\mathbf x\right)=\tilde{{D}}^{-\frac{1}{2}} \tilde{{A}} \tilde{{D}}^{-\frac{1}{2}} \mathbf x .
\end{equation}
\end{proof}
%Then, we can obtain the closed-form solution of Eq.~\ref{opGCN} as the following steps:

The structure learning of the GCN convolution operator can be seen as a graph signal optimization with a laplacian regularization term. Hence, the graph convolution layer tends to \textbf{\textit{increase the feature similarity between the connected nodes}}. Therefore, for heterophilic graphs, this structure learning will make neighbor heterophilic nodes have more similar embeddings, thus leading to poor classification performance. Based on this analysis, we introduce the rationality of our proposed LHS.

\subsection{1.2. The Rationale behind LHS}
In order to mitigate the ``right-shift" phenomenon, which is caused by the GCN structure learning on inherent heterophilic structure, we propose LHS, a robust structure learning method on heterophilic graphs. It has the following two rationales that drew inspiration from the structural OOD problem:

\begin{itemize}
\item R1: For potential homophilic relationships, we aim to obtain the refined structure with more close homophilic connections, thus playing a positive role in the GCN structure learning.
\item R2: For potential heterophilic relationships, our proposed LHS aims to suppress the negative edge connections, thus mitigating the "right-shift" phenomenon. We proposed a truncated threshold GCN to ensure that the negative edges with low confidence can preserve high similarity, thus mitigating the performance deterioration when GCN conducts structure learning.

%can screening the inevitable negative edges by preserving high similarity with target node, thus will mitigate the negative effect of laplace regularization.
%aim that those nodes have high similarity with the connected central node, thus the high-similarity nodes aggregation will mitigate the negative effect of laplace regularization.
\end{itemize}

The proposed LHS designs customized structure learning components to realize the above two rationales, which have not been completely considered by the existing structure learning methods. For example, the previous work~\cite{jin2021universal}~\cite{jin2021node} considered the homophilic structure learning with high feature similarity, while they only utilize the pair-wise information of node features, i.e. kNN networks. Furthermore, they can not handle the potential heterophilic nodes which can hardly filter, while LHS proposes a truncated threshold GCN to only preserve the high-similarity connections, thus mitigating the performance deterioration.

\subsection{1.3. The compatibility of LHS with sampling methods.}
The proposed LHS is compatible with various sampling methods for contrastive learning. The positive instances and the negative instances, which are selected by a sampling method, can be fed into the contrastive learner to generate node representations. We follow the previous work~\cite{zhu2020deep} to use a popular sampling method in the proposed dual-view contrastive learner in our paper.
%And the rationality $R2$ comes from that current heterophilic GNN models inevitably aggregated the different-classes nodes, and based on theoretical analysis, the laplacian regularization between these nodes will break the original node feature and lead to poor performance. It is hard for current potential neighbor search methods to handle this problem, however, LHS paid attention to this question and propose $R2$. Specifically, LHS employed self-expressive to conduct a class-anchored (or class-constrained) global neighbor search, and this is the key factor that makes LHS learned structure more robust. For $R1$, different from the naive feature-level cosine metric (kNN networks), class-anchored self-expressive well captured the class-constraint node similarity, thus the learned structure fits $R1$ more. For $R2$, benefits from the global search and pair-wise constraint, the different-classes neighbors will keep a high similarity to further reduce the performance hurt when conducting aggregation.
\subsection{1.4. The graph tasks that could benefit from the proposed LHS}
We additionally list the following tasks and give underlying reasons as follows.
\begin{itemize}
    \item \textbf{GCN Model robustness} aims to protect the GCN models from various attacks. This task is under-explored on heterophilic graphs. The latent homophilic structure can effectively harden graph models when facing OOD or injected attacks.
    %The latent homophilic structure can effectively harden the application when handling node classification tasks. The robustness is crucial in the nowadays safety-critical world, especially in heterophilic graphs, while the robustness in heterophilic graphs is under-explored by the previous methods.
    \item \textbf{Link prediction} aims to induce the edge relationships between nodes~\cite{xiong2023ground}. The latent homophilic structure learned by LHS captures the homophilic information of nodes and benefits the similar-feature-based link prediction tasks, such as similar product recommendations, and social recommendations with the same interests, etc.
    \item \textbf{High-order graph tasks} aim to make predictions on high-order structures, such as a community~\cite{cavallari2019embedding,qiu2022fast}. The community is considered the smallest unit of prediction in these tasks and is hard to handle by the existing pairwise methods. The proposed latent homophilic structure can benefit these tasks by generating a refined latent structure.
\end{itemize}

\subsection{1.5. Real-world example of the vulnerability of GCN models under malicious threats}
Figure 1 (d) of the main paper demonstrates a malicious attack that can significantly degrade the classification accuracy of a system by $29.30\%$. We refer to the previous work~\cite{zhu2020beyond} that involves real-world examples in fraud detection. We inject attacks, including malicious heterophilic nodes or spurious relations, into target nodes on a real-world social network, e.g., ``the Actor''. Experimental results on ``the Actor'' in Table~1 of the main paper further confirm the vulnerability of GCN models under the above threats. We will follow the reviewer's suggestion to include these examples to enhance the demonstration of the real-world value of our work.

\section{2. The Complexity Analysis of LHS}
We analyze the time and model complexity of LHS and provide some acceleration strategies.
\subsection{2.1. Time Complexity}
The time complexity of LHS mainly comes from the denoising SVD decomposition in the robust structure learner. A naive time complexity is $O(rN^2)$, where $r=4K+1 \ll N$, and $K$ is the number of node classes. In general, the time complexity of LHS is $O(N^2)$, which is more suitable for the small dataset. For large-scale datasets, in order to balance time cost and performance, based on random projection principle~\cite{halko2011finding}, we introduce the randomized SVD to further reduce the time complexity to $O(r\rm {log}\it (N)N)$, which is an acceptable complexity. Then, we can also generate the node representations within a local structure that consists of neighbor nodes instead of the computation throughout the whole graph
%between $O(N)$ and $O(N^2)$.
\subsection{2.2. Model Complexity}
For model complexity, benefit from the sampled learning strategy, for both robust structure learner and re-mask feature augmenter, have high scalability for large-scale datasets, because they only need to employ contrastive learning or reconstructing on those sampled nodes features. One the other hand, the parameters introduced are always transforming the node feature dimension to a smaller dimension. To sum up, LHS only introduces additional parameters that is linear to the feature dimension.

%\subsection{The Pseudocode of LHS Implementation}

\section{3. Structural OOD Evidence}
We provide more empirical evidence on structural OOD as well as the insights gained from interesting observation as follows.
\subsection{3.1. Single-hop Layer}
First, we show the empirical single-hop structural OOD observation on Wisconsin as Fig.~\ref{wis}, Squirrel as Fig.~\ref{squ}, Chameleon as Fig.~\ref{cha}, and the `right-shift' of structural distribution can be seen between training set and test set; this structural OOD will lead to poor robustness. We also study the \textbf{counter example} of heterophilic graphs, as a homophilic graph, Cora can hardly be found the structural OOD problem in both four single-hop layers.
\vspace{-10pt}
\iffalse
\begin{figure*}[h]
  \centering
  \subfigure[]{
  \begin{minipage}[t]{0.23\linewidth}
  \centering
  \includegraphics[scale=0.3]{pictures/Wisconsin/no/layer1.pdf}
  %\caption{fig1}
  \end{minipage}%
  }%
  \subfigure[]{
  \begin{minipage}[t]{0.23\linewidth}
  \centering
  \includegraphics[scale=0.3]{pictures/Wisconsin/no/layer2.pdf}
  %\caption{fig2}
  \end{minipage}%
  }%
  \subfigure[]{
  \begin{minipage}[t]{0.23\linewidth}
  \centering
  \includegraphics[scale=0.3]{pictures/Wisconsin/no/layer3.pdf}
  %\caption{fig2}
  \end{minipage}
  }%
  \subfigure[]{
  \begin{minipage}[t]{0.23\linewidth}
  \centering
  \includegraphics[scale=0.3]{pictures/Wisconsin/no/layer4.pdf}
  %\caption{fig2}
  \end{minipage}
  }%
  \centering
  \vspace{-10pt}
  \caption{The inner-domain structural OOD in Wisconsin.}
  %\label{wis}
\end{figure*}
\fi

%Figure 1
   \begin{figure*}[t]
        \centering
        \subfigure[]{
        \begin{minipage}[t]{0.23\linewidth}
        \centering
        \includegraphics[scale=0.3]{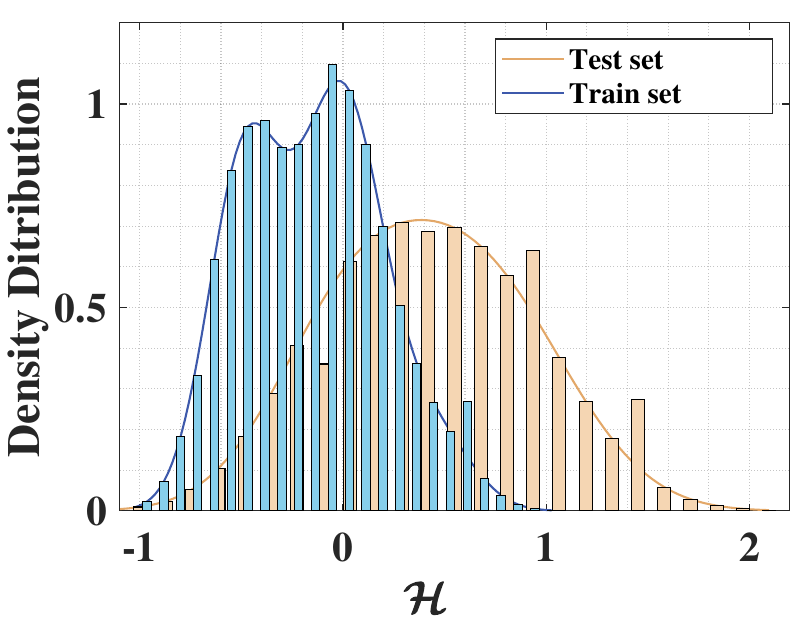}
        %\caption{fig1}
        \end{minipage}%
        }%
        \subfigure[]{
        \begin{minipage}[t]{0.23\linewidth}
        \centering
        \includegraphics[scale=0.3]{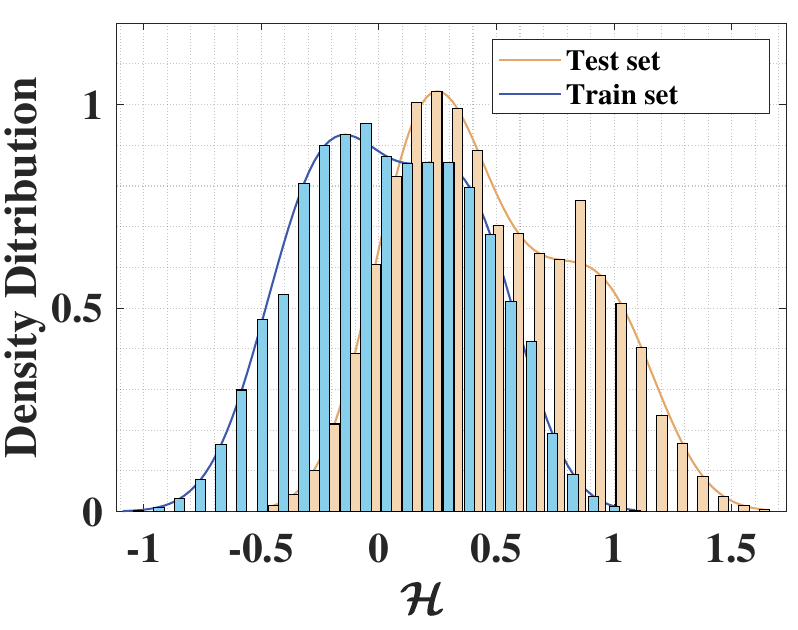}
        %\caption{fig2}
        \end{minipage}%
        }%
        \subfigure[]{
        \begin{minipage}[t]{0.23\linewidth}
        \centering
        \includegraphics[scale=0.3]{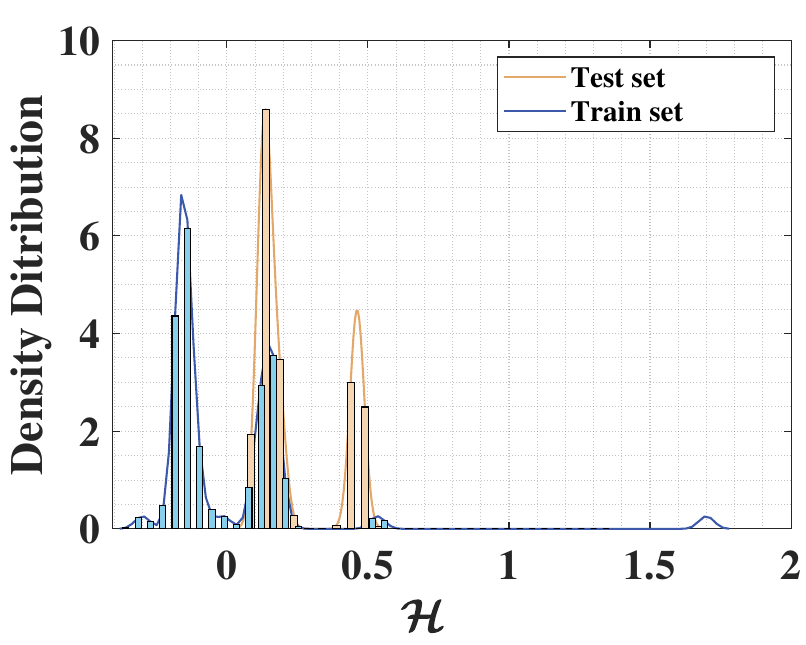}
        %\caption{fig2}
        \end{minipage}
        }%
        \subfigure[]{
        \begin{minipage}[t]{0.23\linewidth}
        \centering
        \includegraphics[scale=0.3]{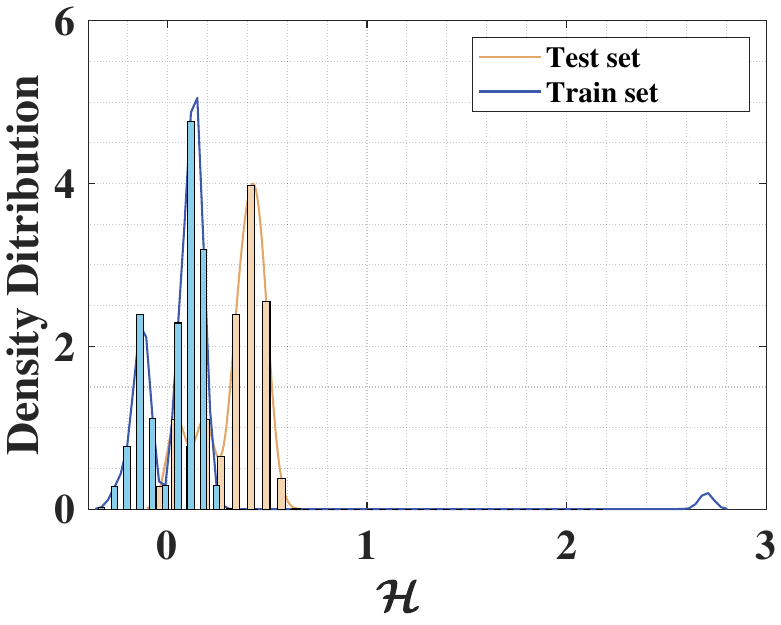}
        %\caption{fig2}
        \end{minipage}
        }%
        \centering
        \vspace{-10pt}
        \caption{The `right-shift' phenomenons on the Wisconsin dataset.}
        \label{wis}
  \end{figure*}
\vspace{-10pt}

%Figure 2
\begin{figure*}[htbp]
  \centering
  %\label(squ}
  \subfigure[]{
  \begin{minipage}[t]{0.23\linewidth}
  \centering
  \includegraphics[scale=0.3]{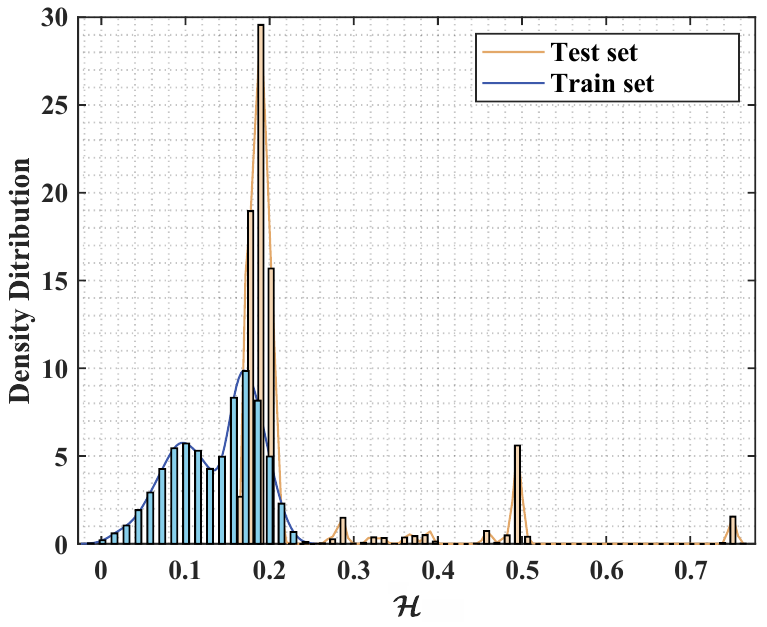}
  %\caption{fig1}
  \end{minipage}%
  }%
  \subfigure[]{
  \begin{minipage}[t]{0.23\linewidth}
  \centering
  \includegraphics[scale=0.3]{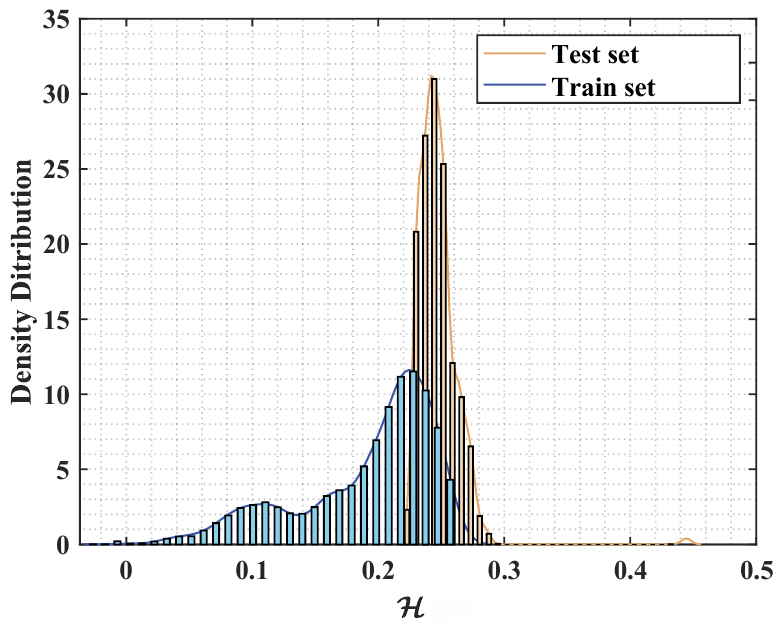}
  %\caption{fig2}
  \end{minipage}%
  }%
  \subfigure[]{
  \begin{minipage}[t]{0.23\linewidth}
  \centering
  \includegraphics[scale=0.3]{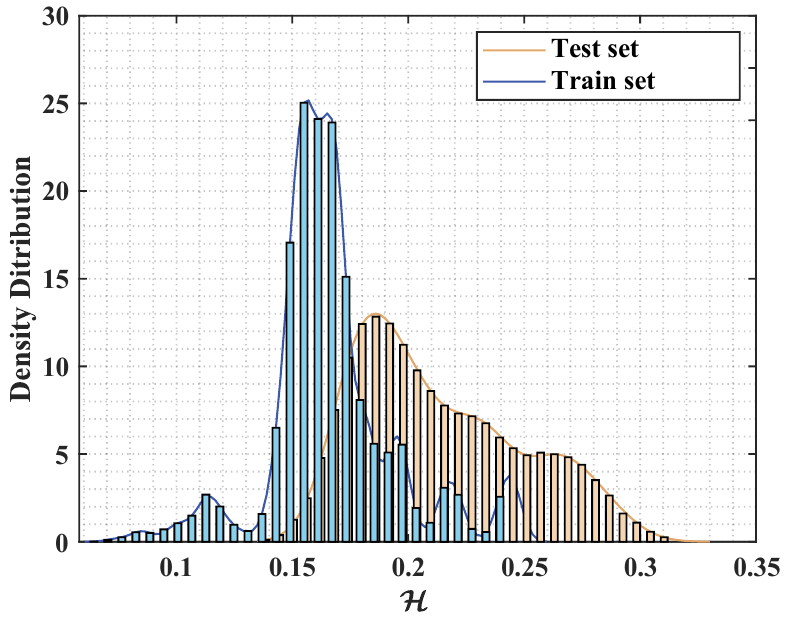}
  %\caption{fig2}
  \end{minipage}
  }%
  \subfigure[]{
  \begin{minipage}[t]{0.23\linewidth}
  \centering
  \includegraphics[scale=0.3]{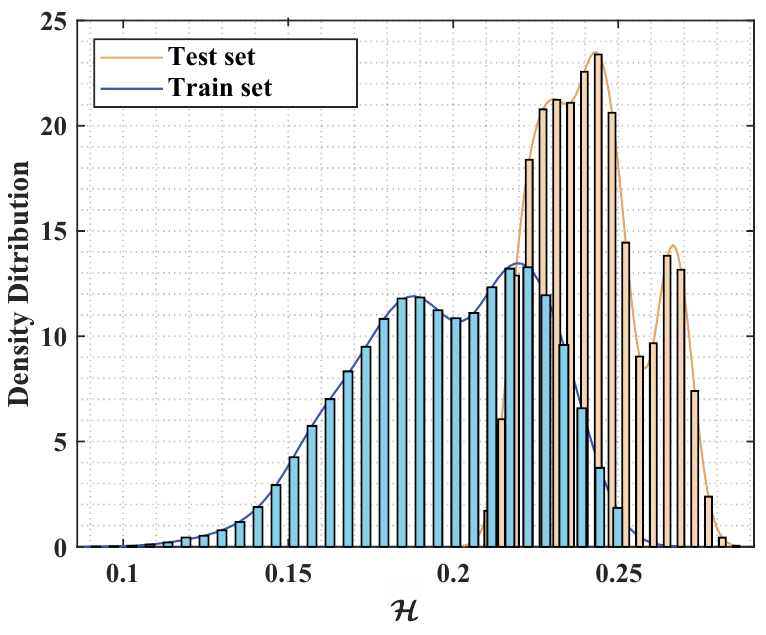}
  %\caption{fig2}
  \end{minipage}
  }%
  \centering
  \vspace{-10pt}
  \caption{The `right-shift' phenomenons on the Squirrel dataset.}
  \label{squ}
 \end{figure*}

%Figure 3
\begin{figure*}[htbp]
%\vspace{-30pt}

  \centering
  \subfigure[]{
  \begin{minipage}[t]{0.23\linewidth}
  \centering
  \includegraphics[scale=0.3]{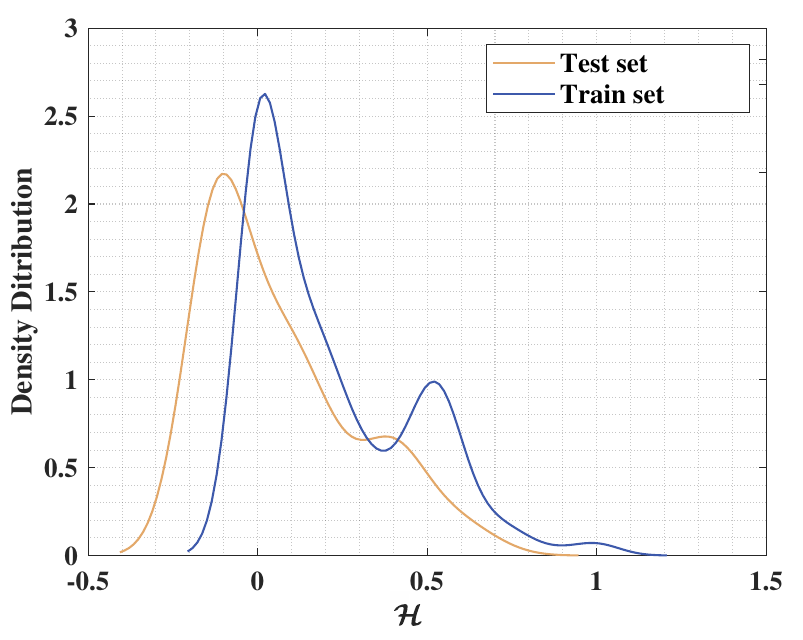}
  %\caption{fig1}
  \end{minipage}%
  }%
  \subfigure[]{
  \begin{minipage}[t]{0.23\linewidth}
  \centering
  \includegraphics[scale=0.3]{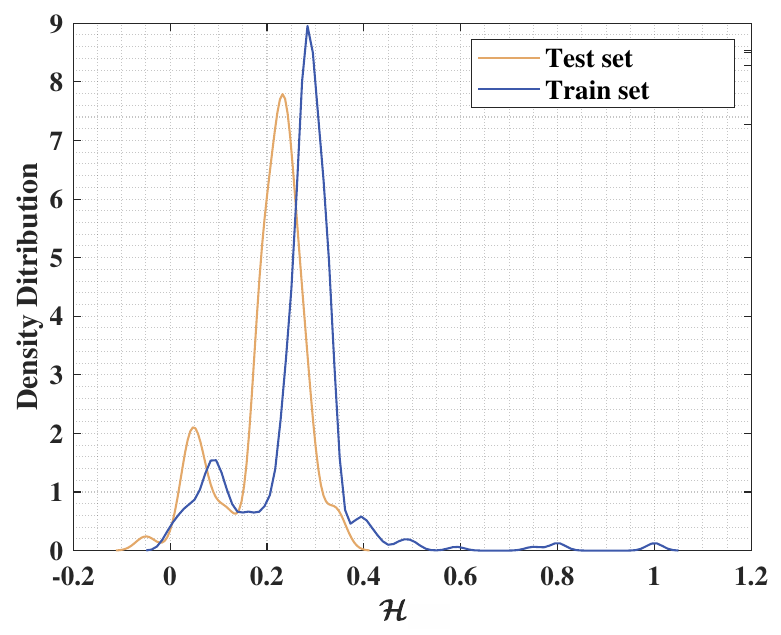}
  %\caption{fig2}
  \end{minipage}%
  }%
  \subfigure[]{
  \begin{minipage}[t]{0.23\linewidth}
  \centering
  \includegraphics[scale=0.3]{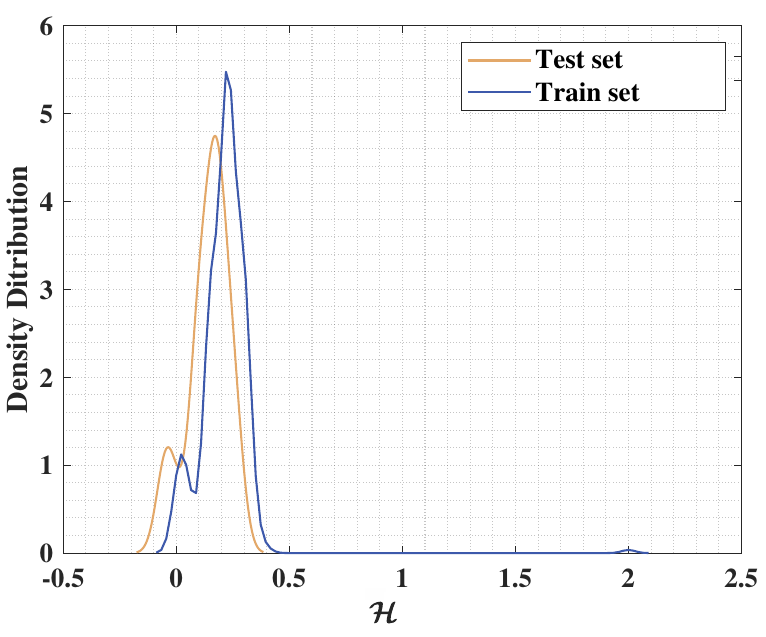}
  %\caption{fig2}
  \end{minipage}
  }%
  \subfigure[]{
  \begin{minipage}[t]{0.23\linewidth}
  \centering
  \includegraphics[scale=0.3]{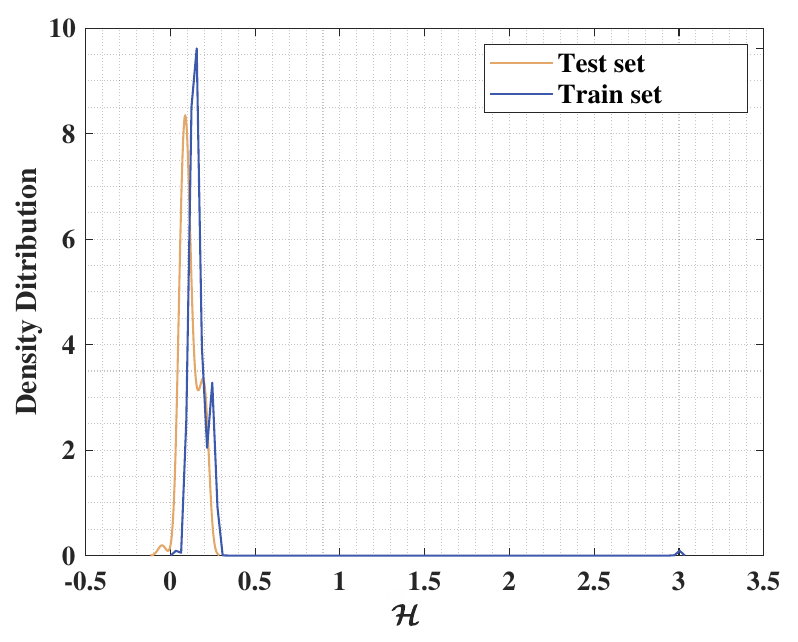}
  %\caption{fig2}
  \end{minipage}
  }%
  \centering
  \vspace{-1em}
  \caption{The `right-shift' phenomenons on the Chameleon dataset.}
  \label{cha}
\end{figure*}

%Figure 4
%A Counter-Example of Structural Distribution OOD
\begin{figure*}[htbp]
  %\vspace{-5em}
  \centering
  \subfigure[]{
  \begin{minipage}[t]{0.23\linewidth}
  \centering
  \includegraphics[scale=0.3]{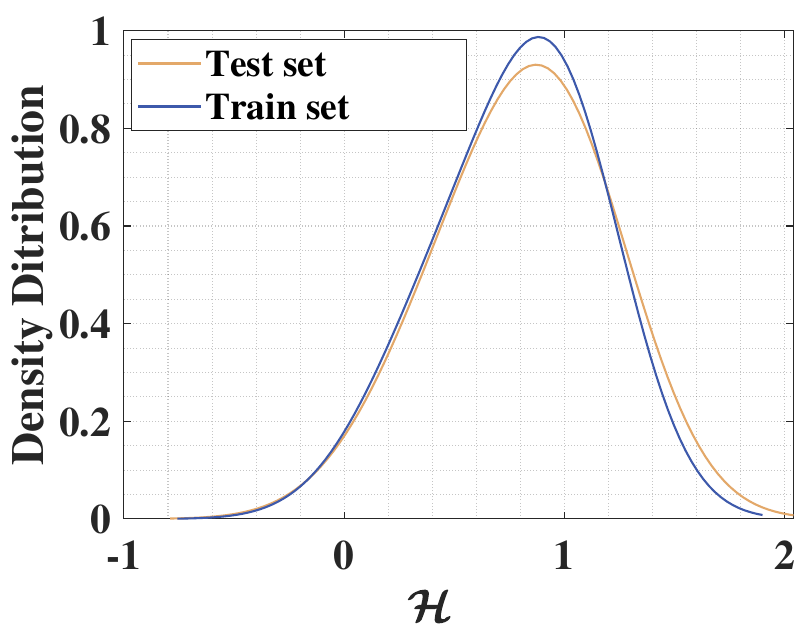}
  %\caption{fig1}
  \end{minipage}%
  }%
  \subfigure[]{
  \begin{minipage}[t]{0.23\linewidth}
  \centering
  \includegraphics[scale=0.3]{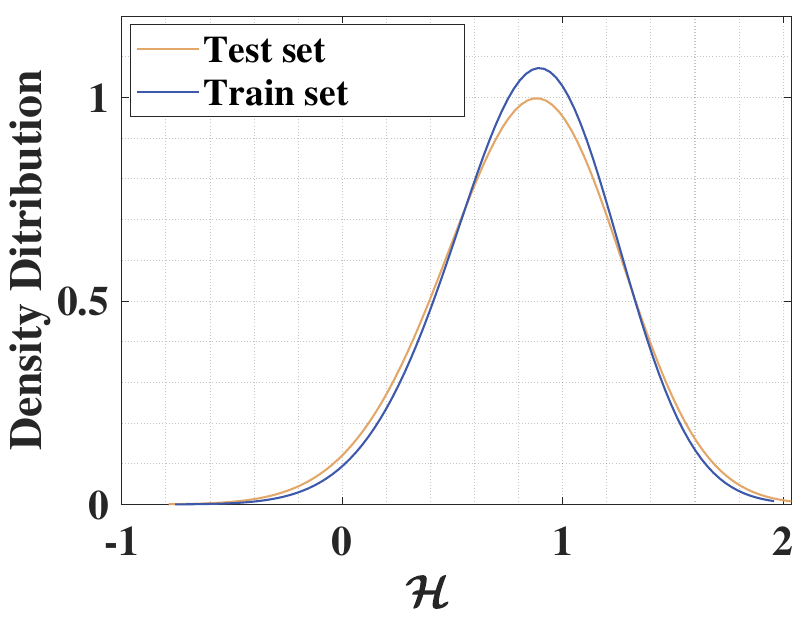}
  %\caption{fig2}
  \end{minipage}%
  }%
  \subfigure[]{
  \begin{minipage}[t]{0.23\linewidth}
  \centering
  \includegraphics[scale=0.3]{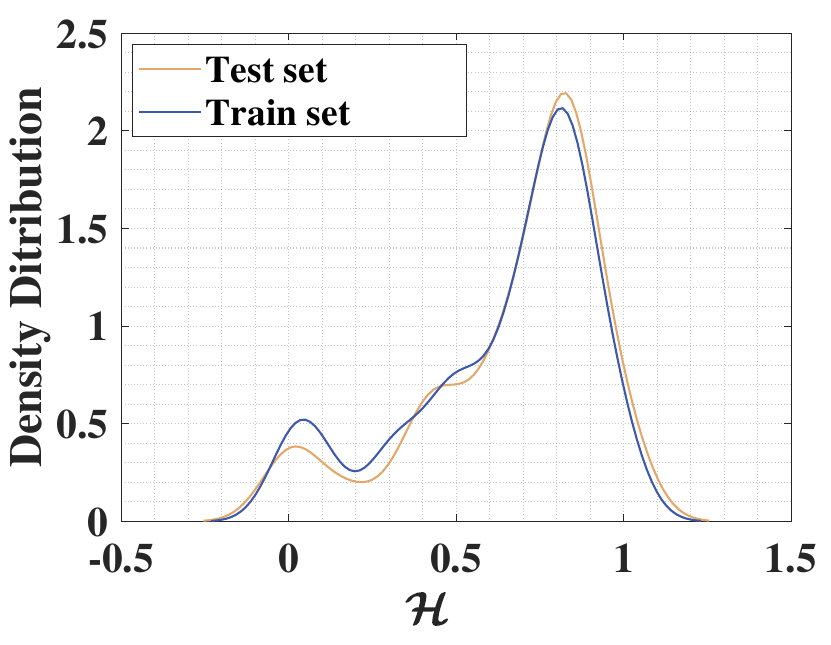}
  %\caption{fig2}
  \end{minipage}
  }%
  \subfigure[]{
  \begin{minipage}[t]{0.23\linewidth}
  \centering
  \includegraphics[scale=0.3]{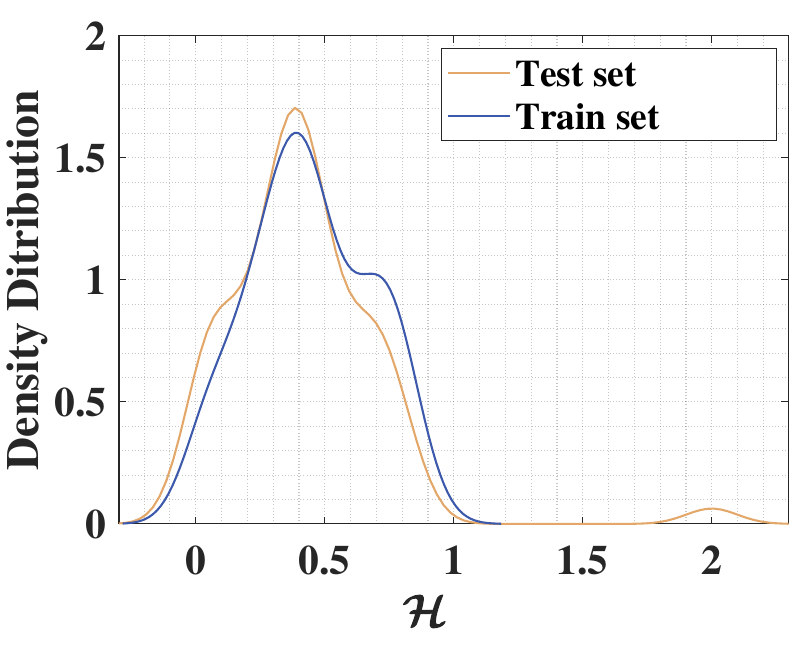}
  %\caption{fig2}
  \end{minipage}
  }%
  \centering
%  \vspace{-10pt}
  % \vspace{-1em}
  \caption{The counter-example Cora dataset, is free of the `right-shift' phenomenons.}
\end{figure*}

\clearpage
\subsection{3.2. Multi-hop Layer}
We also provide the multi-hop joint distribution of edges ($\mathcal H$). The `right-shift' edge distribution becomes more serious on heterophilic graphs that there is always a different degree of shift when each time we generate the distribution. While the homophilic graph Cora hardly exists the distribution shift, even the in multi-hop situation.
\begin{figure*}[h]
  \centering
  \subfigure[Three layer of Cora]{
  \begin{minipage}[t]{0.23\linewidth}
  \centering
  \includegraphics[scale=0.3]{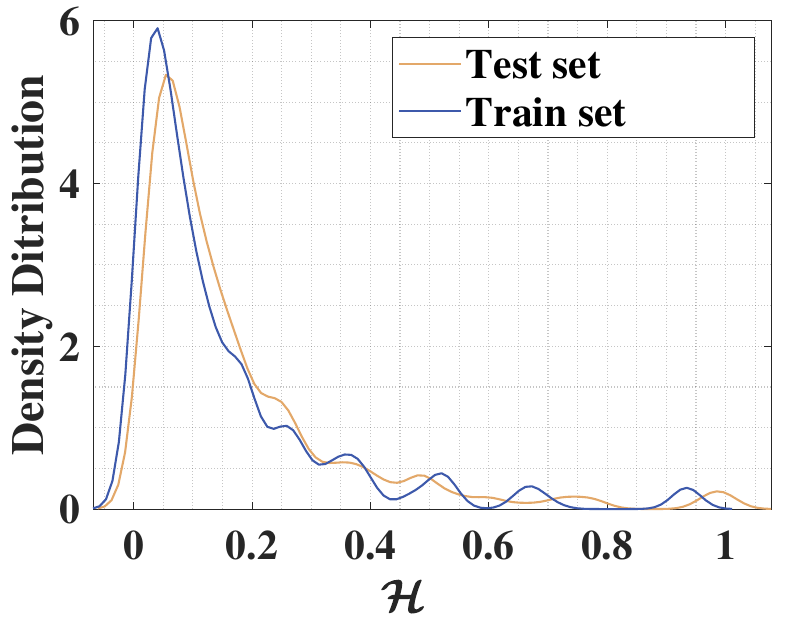}
  %\caption{\small Three layer of Cora}
  \end{minipage}%
  }%
  \subfigure[Three layer of Squirrel]{
  \begin{minipage}[t]{0.23\linewidth}
  \centering
  \includegraphics[scale=0.15]{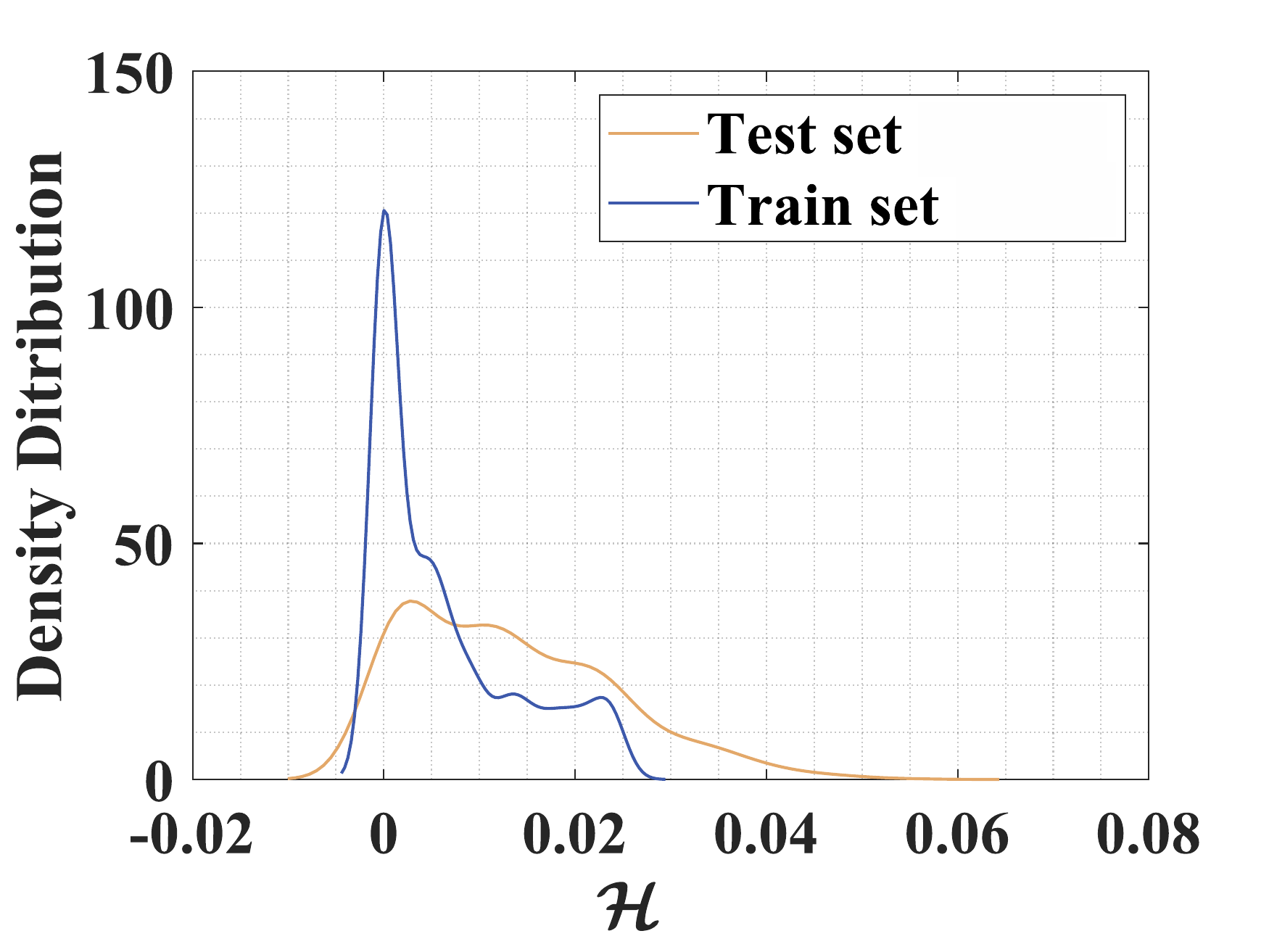}
  %\caption{\small Three layer of Squirrel}
  \end{minipage}%
  }%
  \subfigure[Three layer of Chameleon]{
  \begin{minipage}[t]{0.23\linewidth}
  \centering
  \includegraphics[scale=0.3]{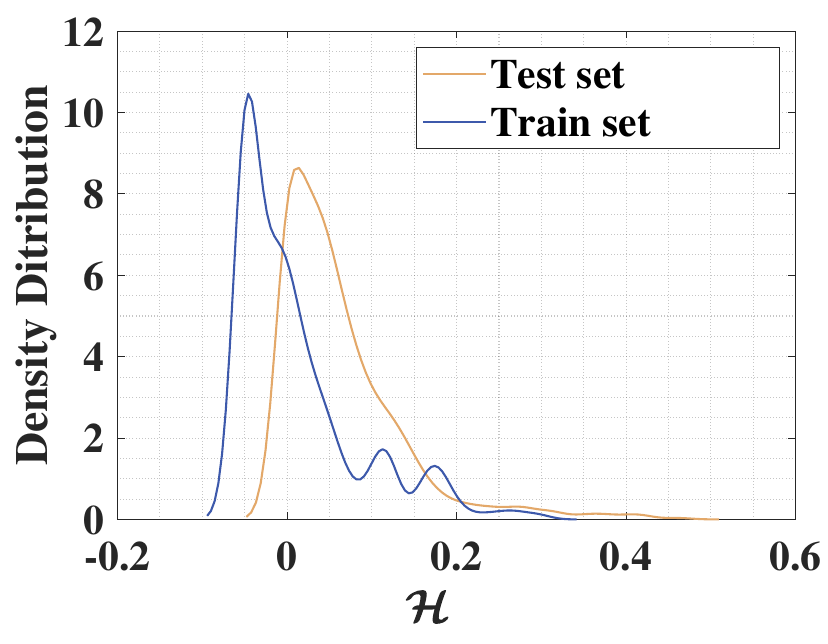}
  %\caption{\small Three layer of Chameleon}
  \end{minipage}
  }%
  \subfigure[Three layer of Wisconsin]{
    \begin{minipage}[t]{0.23\linewidth}
    \centering
    \includegraphics[scale=0.3]{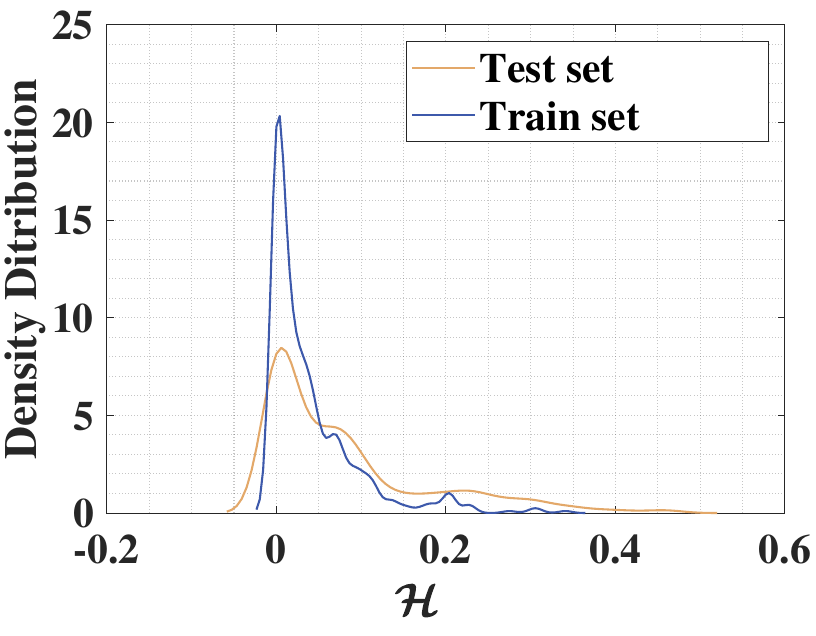}
    \end{minipage}
    }%

  \centering
  \caption{The multi-hop edge distribution on different datasets.}
\end{figure*}

\section{4. Experiments Details}
\subsection{4.1. Datasets Information}
The statistical information about the eight real-world networks is given in Table~\ref{Inf}, including the number of classes, the dimension of features, and the number of nodes and edges.
\begin{table*}[h]
\centering
\caption{Statistical information of eight real-world datasets.}
  \begin{tabular}{c|cccccccc}
  Dataset & Cora & Citeseer & PubMed & Chameleon & Squirrel & Actor & Texas & Cornell \\ \hline
  Classes & $7$ & $6$ & $5$ & $5$ & $5$ & $5$ & $5$ &$5$ \\
  Features & $1433$ & $3703$ & $500$ & $2325$ & $2089$ & $932$ & $1703$ & $1703$ \\
  Nodes & $2708$ & $3327$ & $19717$ & $2277$ & $5201$ & $7600$ & $183$ & $183$ \\
  Edges & $5278$ & $4552$ & $44324$ & $31371$ & $198353$ & $26659$ & $279$ & $277$
  \end{tabular}\label{Inf}
  \end{table*}

\subsection{4.2. Split ratio of dataset.}
For fair comparisons, we follow the settings of previous work~\cite{Pei2020Geom-GCN:} to use 60\%, 20\%, and 20\% of the data for training, validation, and testing, respectively. Table 1 and Table 2 of the main paper report the comparison results under such settings.

\subsection{4.3. Experiment Hardware and Setting}
\subsubsection{Hardware}
All experiments are performed on a Linux Machine with Intel(R) Xeon(R) Gold $6330$ CPU with $28$ cores, $378$GB of RAM, and a NVIDIA
RTX$3090$ GPU with $24$GB of GPU memory.
\subsubsection{Additional Setting}
We use Adam Optimizer with
the $\beta_{1}=0.9,\beta_{2}=0.999,\epsilon=1 \times 10^{-8}$ on all datasets.
We run epochs $\in \{300, 1000\}$ with the learning rate $\in \{0.000005, 0.0001\}$ and apply early stopping with a patience of $40$.
We use PReLU as activation and the scaling factor $\tau \in$ ${0, 2}$. According to the cross-entropy loss and ACC, we get a best model.
On the validation set, we set hidden unit $\in \{16, 32\}$, learning rate $\in \{0.00005, 0.0001\}$, dropout in each layer $\in \{0.1, 0.2\}$
and the weight decay is $5e-4$. Our method is implemented by Pytorch and Pytorch Geometric.

\section{5. The Limitation of Previous Structure Learning Strategy}
\label{sec:limitation-structure-learning}
A large amount of current structure learning methods focus on the feature-level similarity to conduct structure learning~\cite{jin2021universal}\cite{jin2021node}, which always formed poor similarity and can not reflect the true relationship between node pairs.
\begin{equation}
\mathbf{S}_{i j}=\frac{\mathbf{x}_{i}^{\top} \mathbf{x}_{j}}{\left\|\mathbf{x}_{i}\right\|\left\|\mathbf{x}_{j}\right\|}
\end{equation}
Now we provide empirical evidence to support our claim. To sum up, following the kNN method, we have calculated the cosine similarity on both homophilic graphs~\cite{yang2016revisiting} (Cora, Citeseer, Pubmed) and heterophilic graphs~\cite{Pei2020Geom-GCN:} (Texas, Wiscon, Squirrel) as follows. We found that almost all node pairs that have formed bad similarities are only between $0.1$ and $0.5$. More seriously, there is no significant similarity difference between the same and different classes of nodes, whether they are between pairs or average levels. The results reveal that we can hardly distinguish nodes in different classes under such metrics and thus fail to learn a better structure to mitigate the ``right-shift" phenomenon.

We respectively generate structures via the kNN network method and our proposed LHS. First, we compare the overall average similarity of the structures generated by the two methods; as shown in Fig.~\ref{simcora} to Fig.~\ref{simcha} (a), the average similarity of the LHS is between $0.6044$ to $0.7844$ on each dataset, while the average similarity of kNN networks only is $0.0562$ to $0.3451$. The (b) of Fig.~\ref{simcora} to Fig.~\ref{simcha} is the average similarity of sampled node and its homophilic (left) and heterophilic (right) neighbors (we sampled the same and enough neighbors to calculate the average similarity); and under the average level, the homophilic similarities are always higher than the heterophilic similarities constructed by RSGC (left), while getting closer to the ground-truth similarity. However, the similarity constructed by the kNN network (right) is poor and the heterophilic similarities are even higher than the homophilic similarities in Cora, Squirrel, and Chameleon, which makes the heterophilic nodes harder to distinguish, thus led poor structure learning effect. Note that the (b) from Fig.~\ref{simcora} to Fig.~\ref{simcha} represent the statistical similarity from the LHS (dark blue) and kNN network (light blue).

%\section{How Powerful of Structural Learning?}

\begin{figure}[H]
    \centering
    \subfigure[]{
    \begin{minipage}[t]{0.4\linewidth}
    \centering
    \includegraphics[width=6cm]{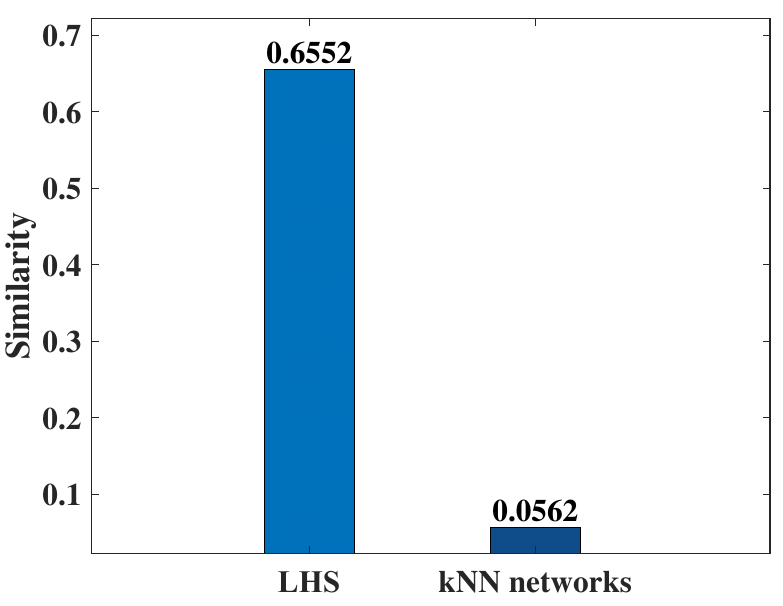}
    %\caption{fig1}
    \end{minipage}%
    }%
    \subfigure[]{
    \begin{minipage}[t]{0.4\linewidth}
    \centering
    \includegraphics[width=6cm]{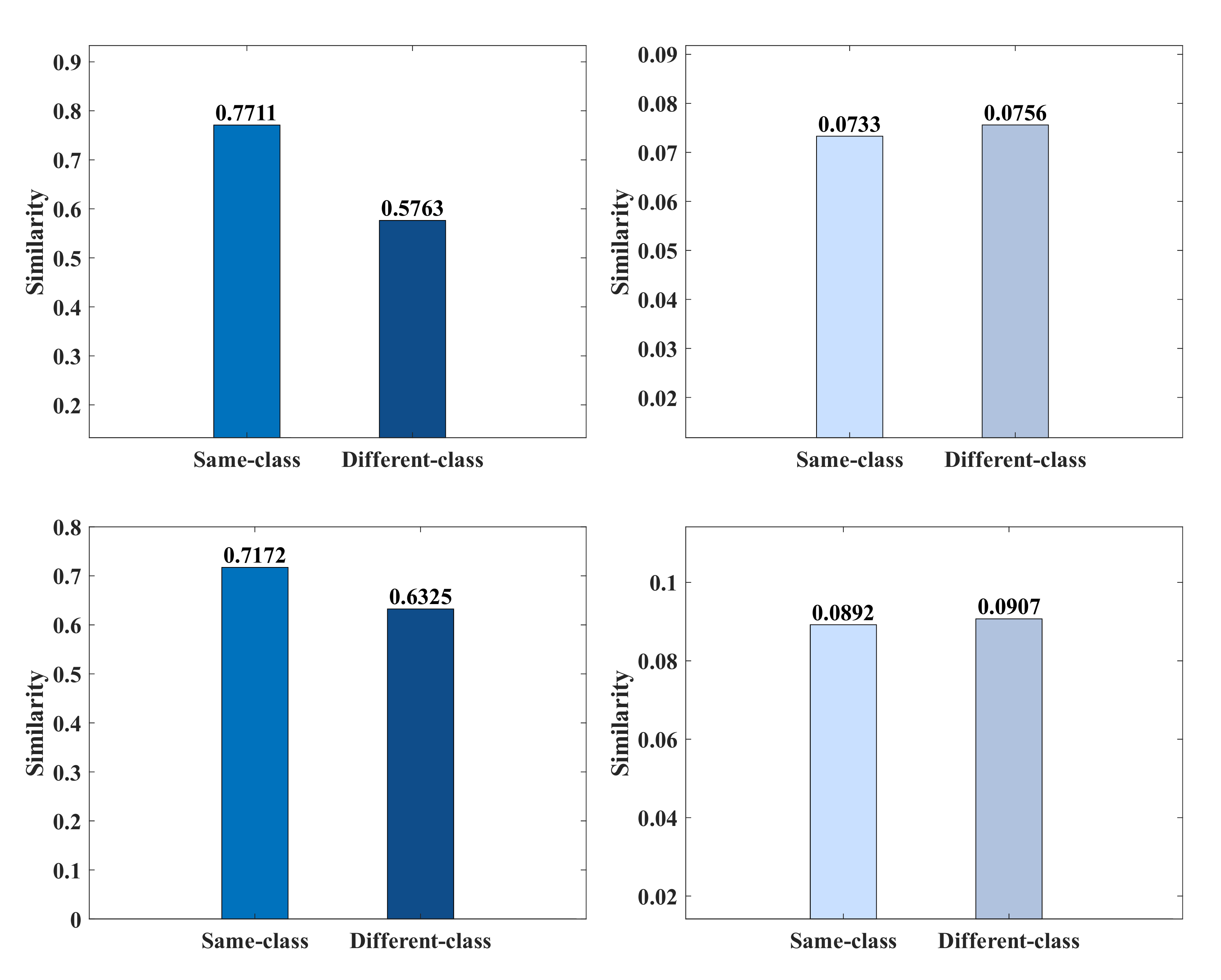}
    %\caption{fig2}
    \end{minipage}%
    }%
    \centering
    %\label(cha}
    \vspace{-15pt}
    \caption{Similarity comparison on Cora.}
    \label{simcora}
  \end{figure}

%texas
  \begin{figure}[H]
    \vspace{-20pt}
      \centering
      \subfigure[]{
      \begin{minipage}[t]{0.4\linewidth}
      \centering
      \includegraphics[width=6cm]{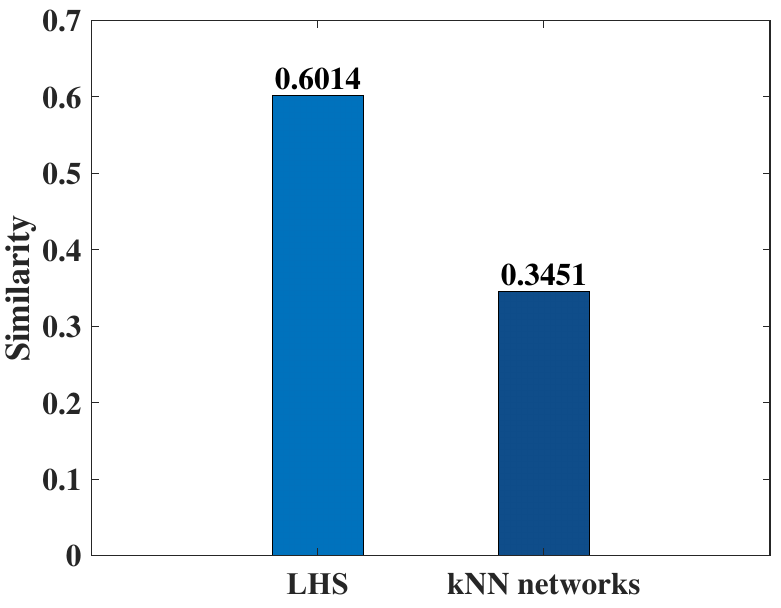}
      %\caption{fig1}
      \end{minipage}%
      }%
      \subfigure[]{
      \begin{minipage}[t]{0.4\linewidth}
      \centering
      \includegraphics[width=6cm]{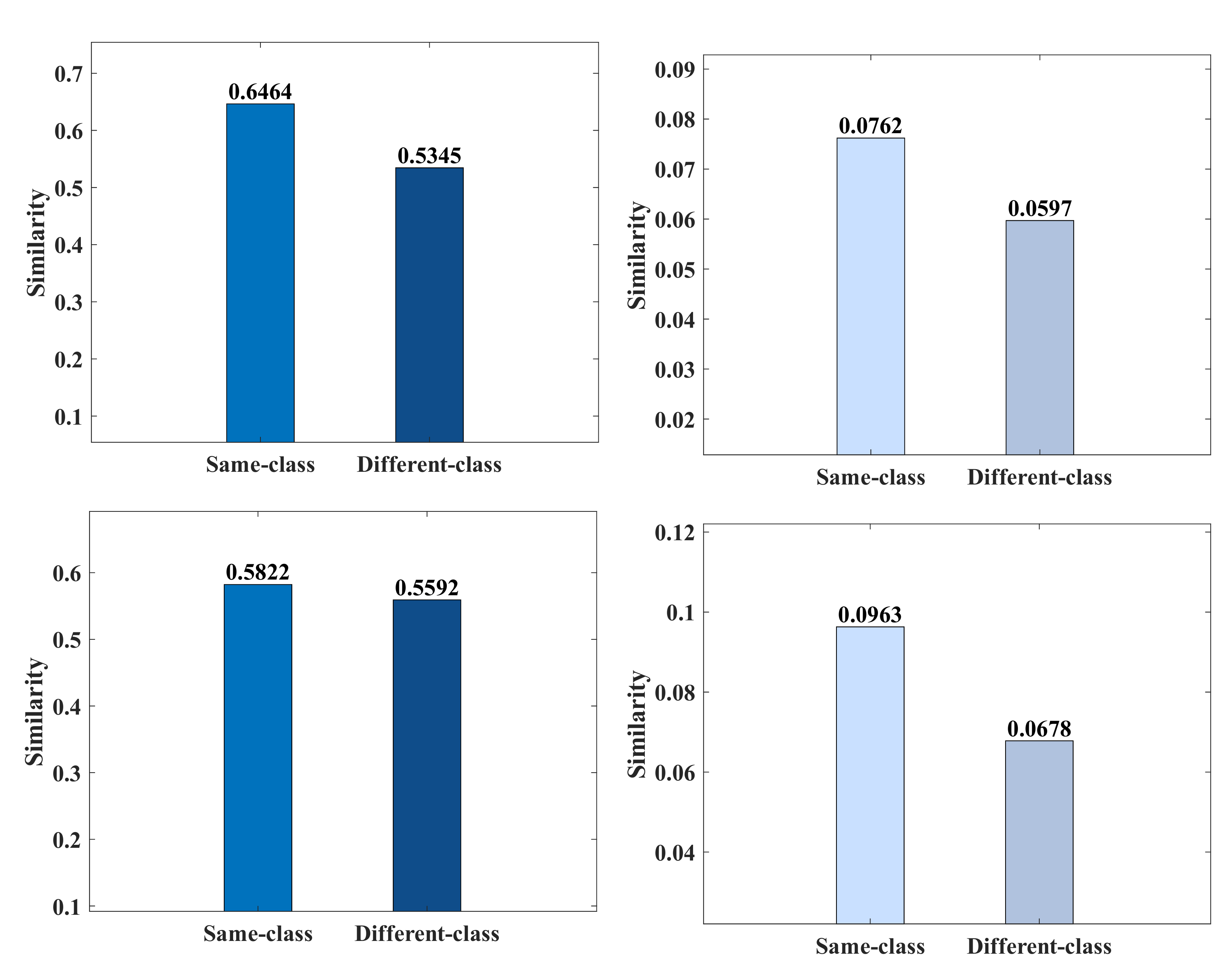}
      %\caption{fig2}
      \end{minipage}%
      }%
      \centering
      %\label(cha}
      \vspace{-15pt}
      \caption{Similarity comparison on Texas.}
      \label{simtex}
  \end{figure}

%squirrel
  \begin{figure}[H]
    \vspace{-20pt}
      \centering
      \subfigure[]{
      \begin{minipage}[t]{0.4\linewidth}
      \centering
      \includegraphics[width=6cm]{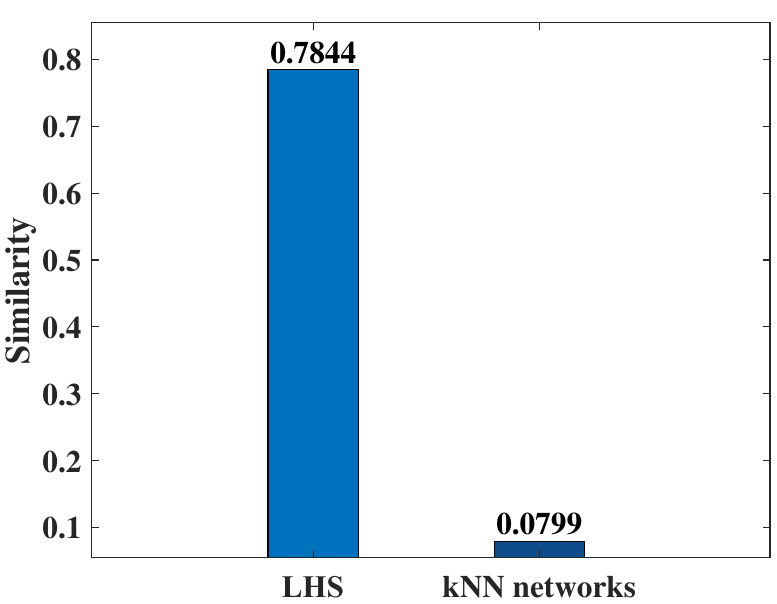}
      %\caption{fig1}
      \end{minipage}%
      }%
      \subfigure[]{
      \begin{minipage}[t]{0.4\linewidth}
      \centering
      \includegraphics[width=6cm]{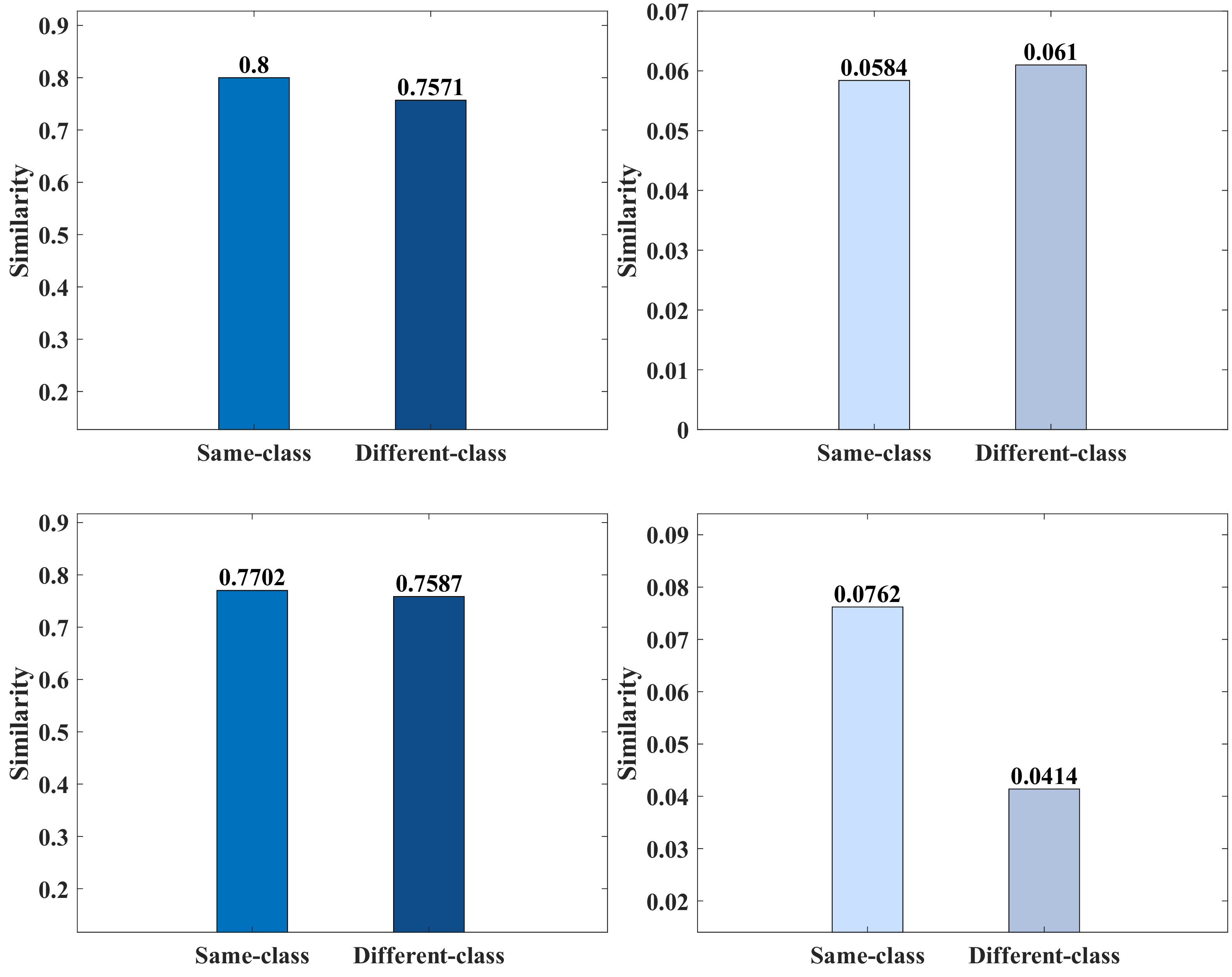}
      %\caption{fig2}
      \end{minipage}%
      }%
      \centering
      %\label(cha}
      \vspace{-15pt}
      \caption{Similarity comparison on Squirrel.}
      \label{simsqu}
  \end{figure}

%chameleon
  \begin{figure}[H]
    %\vspace{-20pt}
      \centering
      \subfigure[]{
      \begin{minipage}[t]{0.4\linewidth}
      \centering
      \includegraphics[width=6cm]{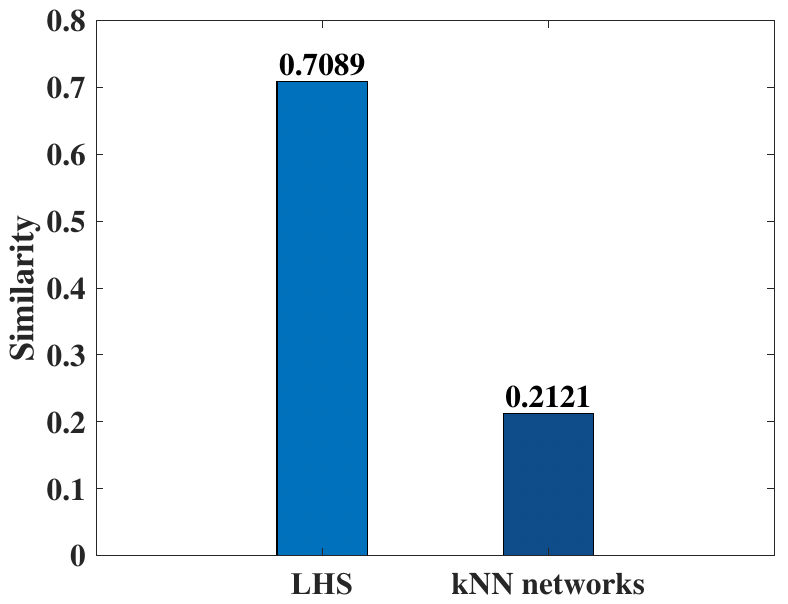}
      %\caption{fig1}
      \end{minipage}%
      }%
      \subfigure[]{
      \begin{minipage}[t]{0.4\linewidth}
      \centering
      \includegraphics[width=6cm]{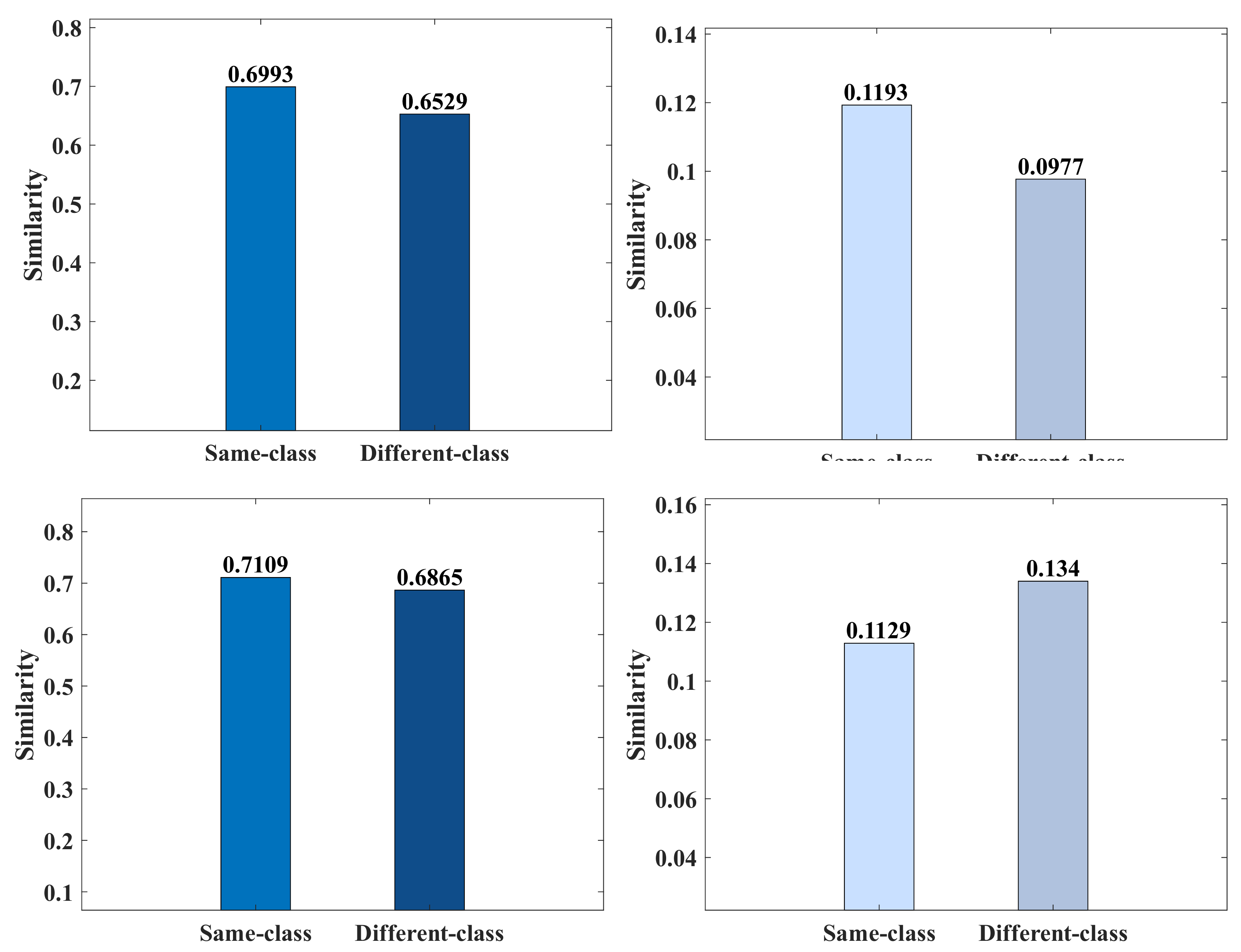}
      %\caption{fig2}
      \end{minipage}%
      }%
      \centering
      %\label(cha}
      \vspace{-15pt}
      \caption{Similarity comparison on Chameleon.}
      \label{simcha}
    \end{figure}

\end{document}